\documentclass{article}

\usepackage{microtype}
\usepackage{graphicx}
\usepackage{booktabs} 
\usepackage{subfig}
\usepackage{multirow}
\usepackage{makecell}
\usepackage{hyperref}
\usepackage{adjustbox}

\usepackage[accepted]{icml2024}

\usepackage{amsmath}
\usepackage{amssymb}
\usepackage{mathtools}
\usepackage{amsthm}

\usepackage{nicefrac} 
\usepackage{amsfonts} 
\usepackage{enumitem}
\usepackage{graphicx}

\usepackage[capitalize,noabbrev]{cleveref}

\theoremstyle{plain}
\newtheorem{theorem}{Theorem}

\newtheorem{lemma}[theorem]{Lemma}

\theoremstyle{definition}
\newtheorem{definition}[theorem]{Definition}
\newtheorem{assumption}{Assumption}
\theoremstyle{remark}

\usepackage[textsize=tiny]{todonotes}
\usepackage[utf8]{inputenc} 
\usepackage{xr-hyper}
\usepackage[T1]{fontenc}    
\usepackage{hyperref}       
\usepackage{url}            
\usepackage{booktabs}       
\usepackage{amsfonts}       
\usepackage{nicefrac}       
\usepackage{microtype}      
\usepackage{wrapfig}
\usepackage[algo2e,ruled]{algorithm2e}
\usepackage{algorithmic}
\usepackage{subfig}
\usepackage{comment}
\usepackage{graphicx}
\usepackage{bbm}
\usepackage{url}
\usepackage{xr}
\usepackage{multirow}
\usepackage{listings}
\usepackage{graphicx}
\usepackage{float}
\usepackage{subfloat}
\usepackage{cancel}

\usepackage{subfiles}
\definecolor{mypink1}{rgb}{0.858, 0.188, 0.478}

\newtheorem*{rep@theorem}{\rep@title}
\newcommand{\newreptheorem}[2]{%
\newenvironment{rep#1}[1]{%
 \def\rep@title{#2 \ref{##1}}%
 \begin{rep@theorem}}%
 {\end{rep@theorem}}}
\makeatother

\DeclareMathOperator{\E}{\mathbb{E}}
\renewcommand{\P}{\operatorname{\mathbb{P}}}

\DeclarePairedDelimiter{\norm}{\lVert}{\rVert}
\DeclarePairedDelimiter{\abs}{\lvert}{\rvert}

\DeclarePairedDelimiter{\braces}{\{}{\}}
\DeclarePairedDelimiter{\parens}{(}{)}
\DeclarePairedDelimiter{\brackets}{[}{]}
\DeclarePairedDelimiterX{\ip}[2]{\langle}{\rangle}{#1,#2}

\DeclarePairedDelimiterXPP{\normsub}[2]{}{\lVert}{\rVert}{_{#2}}{#1}
\DeclarePairedDelimiterXPP{\ipsub}[3]{}{\langle}{\rangle}{_{#3}}{#1,#2}

\DeclarePairedDelimiterXPP{\ipHS}[2]{}{\langle}{\rangle}{_{\mathrm{HS}}}{#1, #2}
\DeclarePairedDelimiterXPP{\normHS}[1]{}{\lVert}{\rVert}{_{\mathrm{HS}}}{#1}

\DeclarePairedDelimiterXPP{\ipF}[2]{}{\langle}{\rangle}{_{\mathrm{F}}}{#1, #2}
\DeclarePairedDelimiterXPP{\normF}[1]{}{\lVert}{\rVert}{_{\mathrm{F}}}{#1}

\DeclarePairedDelimiterXPP{\normt}[1]{}{\lVert}{\rVert}{_{2}}{#1}
\DeclarePairedDelimiterXPP{\normo}[1]{}{\lVert}{\rVert}{_{\mathrm{1}}}{#1}

\DeclarePairedDelimiterXPP{\dkl}[2]{\operatorname{D_{KL}}}{(}{)}{}{#1 \: \delimsize\Vert \: #2}

\DeclarePairedDelimiterXPP{\restr}[2]{}{{}}{\vert}{_{#2}}{#1}

\newcommand{\R}{\mathbb{R}}

\newcommand{\given}{\:\vert\:}

\newcommand{\simiid}{\overset{\mathclap{\text{i.i.d.}}}{\sim}}

\newcommand{\blx}{\mathbf{x}}
\newcommand{\blz}{\mathbf{z}}
\newcommand{\blI}{\mathbf{I}}

\newcommand{\blv}{\mathbf{v}}
\newcommand{\blu}{\mathbf{u}}
\newcommand{\blw}{\mathbf{w}}
\newcommand{\blZ}{\mathbf{Z}}
\newcommand{\blg}{\mathbf{g}}
\newcommand{\blh}{\mathbf{h}}

\newcommand{\bltheta}{\boldsymbol{\theta}}
\newcommand{\blTheta}{\boldsymbol{\Theta}}
\newcommand{\blthetatl}{\tilde{\bltheta}}
\newcommand{\blthetahat}{\hat{\bltheta}}
\newcommand{\blTStl}{\tilde{\bltheta}^*}
\newcommand{\blTS}{\bltheta^*}

\newcommand{\blrho}{\boldsymbol{\rho}}
\newcommand{\blrhotl}{\tilde{\blrho}}
\newcommand{\bllambda}{\boldsymbol{\lambda}}

\newcommand{\ones}{\mathbf{1}}

\newcommand{\scrL}{\mathcal{L}}
\newcommand{\blSigma}{\boldsymbol{\Sigma}}
\newcommand{\scrN}{\mathcal{N}}
\newcommand{\scrW}{\mathcal{W}}
\newcommand{\scrM}{\mathcal{M}}
\newcommand{\scrS}{\mathcal{S}}

\icmltitlerunning{Unveiling Class Disparities with Spectral Imbalance}

\begin{document}

\twocolumn[
\icmltitle{
Balanced Data, Imbalanced Spectra: \\ Unveiling Class Disparities with Spectral Imbalance
}



\icmlsetsymbol{equal}{*}

\begin{icmlauthorlist}
\icmlauthor{Chiraag Kaushik}{equal,gt}
\icmlauthor{Ran Liu}{equal,gt}
\icmlauthor{Chi-Heng Lin}{gt,comp}
\icmlauthor{Amrit Khera}{gt}
\icmlauthor{Matthew Y Jin}{sch}
\icmlauthor{Wenrui Ma}{gt}
\icmlauthor{Vidya Muthukumar}{gt}
\icmlauthor{Eva L Dyer}{gt}
\end{icmlauthorlist}

\icmlaffiliation{gt}{Georgia Institute of Technology, Georgia, the USA.}
\icmlaffiliation{comp}{Samsung Research.}
\icmlaffiliation{sch}{Stanford University, California, the USA}

\icmlcorrespondingauthor{Ran Liu}{rliu361@gatech.edu}
\icmlcorrespondingauthor{Chiraag Kaushik}{ckaushik7@gatech.edu}
\icmlcorrespondingauthor{Eva L Dyer}{evadyer@gatech.edu}
\icmlcorrespondingauthor{Vidya Muthukumar}{vmuthukumar8@gatech.edu}

\icmlkeywords{Machine Learning, ICML}

\vskip 0.3in
]


\printAffiliationsAndNotice{\icmlEqualContribution} 

\begin{abstract}
Classification models are expected to perform equally well for different classes, yet in practice, there are often large gaps in their performance. This issue of class bias is widely studied in cases of datasets with {\em sample imbalance}, but is relatively overlooked in balanced datasets.
In this work, we introduce the concept of {\em spectral imbalance} in features as a potential source for class disparities and study the connections between spectral imbalance and class bias in both theory and practice.  To build the connection between spectral imbalance and class gap, we develop a theoretical framework for studying class disparities and derive exact expressions for the per-class error in a high-dimensional mixture model setting. We then study this phenomenon in 11 different state-of-the-art pretrained encoders, and show how our proposed framework can be used to compare the quality of encoders, as well as evaluate and combine \emph{data augmentation} strategies to mitigate the issue.
Our work sheds light on the class-dependent effects of learning, and provides new insights into how state-of-the-art pretrained features may have unknown biases that can be diagnosed through their spectra. Code can be found at \href{https://github.com/nerdslab/SpectraImbalance}{https://github.com/nerdslab/SpectraImbalance}.
\end{abstract}

\vspace{-3mm}
\section{Introduction}
A core objective in machine learning is to build fair classification models that provide good performance irrespective of the class to which the data belongs. In practice, however, 
models are often biased towards better performance on some classes of the data and may do a poor job on others. 

Most studies and approaches for dealing with class bias have focused on \textit{sample imbalance}, or the fact that different classes often have different number of samples \cite{zhang2023deep}. In this case, mechanisms to reweight the loss or rebalance the dataset can be used to correct the sample imbalance~\cite{ren2020balanced,zhu2022balanced}. 
Somewhat surprisingly, \emph{even without sample imbalance} there can still be significant performance gaps across classes, including on state-of-the-art encoders \cite{ma2022delving, ma2023curvature}. These effects can become more pronounced depending on the types of regularization and/or data augmentation used during training \cite{balestriero2022effects,kirichenko2023understanding}. 
Thus, we need a better understanding of the sources of class disparities and new approaches for identifying and mitigating underlying biases in models. This is particularly relevant due to the popularity of foundation and pretrained models, since underlying biases in the pretrained features may impact performance and robustness on downstream tasks.

In this work, we introduce a new framework for studying class-dependent generalization that relies on a concept that we call {\em spectral imbalance}. The central idea behind this perspective is that differences, or shifts, in the distribution of features across classes could be the source of class biases. We characterize these differences in  feature geometry using the distribution of \emph{eigenvalues}, i.e. the \emph{spectrum}, for each class and show, both in theory and in practice, that deviations in spectra across classes will indeed produce interesting effects on the class gap. Our findings complement recent empirical work that connects certain geometric properties of learned features to per-class performance \cite{ma2022delving, ma2023curvature} while also providing a formal theoretical connection between spectral imbalance and class gaps in high-dimensional linear models.

First, we develop a theoretical framework for studying this phenomenon and dig deeper into the distinct ways in which spectral imbalance may impact class gap.   In particular, we derive \textit{exact asymptotic expressions} for the class-wise error 
in a Gaussian mixture model with class-specific covariances. 
We use these expressions to uncover different types of spectral imbalance and characterize their consequences for the class gap. This theory provides a strong basis for the investigation of the spectra in per-class generalization.


Next, in extensive empirical investigations, we study the class-dependent spectra for 11 state-of-the-art pretrained encoders (with 6 distinct types of architecture) on ImageNet. 
Across the board, we observe that the class gap is strongly correlated with the spectral imbalance of the representation space and that through measuring different statistics of this imbalance, we can predict which classes will perform worse than others without actually testing the model.
Furthermore, we define a classwise \emph{spectral quantile measure} that we use to diagnose different models, predict which encoders will have smaller class gaps, and also determine smarter augmentation strategies. 
Our empirical results provide strong evidence that spectral imbalance plays an important role in the characteristics of representation space.

In summary, our major contributions are as follows:
\begin{itemize}
\itemsep0em
    \item In Section~\ref{sec:theory-intuition}, we introduce the concept of {\em spectral imbalance} and study how discrepancies in eigenvalues between classes are related to the problem of class-dependent generalization.
    \item In Section~\ref{sec:theory-formal}, we derive \textit{exact} theoretical expressions for the per-class generalization in a high-dimensional linear mixture-model setting (Theorem~\ref{thm:main_thm}). The resultant simulations demonstrate the effect of different spectral imbalances on the per-class performance gap. 
    \item In Section~\ref{sec:experiments}, we empirically demonstrate that similar phenomena hold in the representation space of pre-trained models. Thus, we design a \emph{spectral quantile score}, that accurately measures the amount of `imbalance' in representation space.
    \item Finally, in Section~\ref{exp:data_augmentation}, we provide an initial exploration of how to improve augmentation design to mitigate spectral imbalance, and design a simple ensemble method to combine augmentations and improve performance across classes \textit{without any re-training}.
\end{itemize}


\section{Why spectral imbalance?}\label{sec:theory-intuition}


Effectively addressing class bias necessitates a deep understanding of the underlying 
structure of the representation space.
One promising avenue is through analyzing the spectral properties of the features corresponding to each class. The spectrum, i.e., the set of eigenvalues of the covariance matrix of the representation, serves as a powerful indicator of the features' intrinsic geometry and dimensionality. 
It characterizes how variance is distributed across different principal components, offering a window into the structure of the representation space. 
We thus hypothesize that variations in these 
spectral properties will provide crucial insights into class-specific performance, potentially revealing biases that are not immediately discernible through traditional measures of imbalance (e.g. sample size per class).

We explore in particular the extent to which {\em spectral imbalance}, or variation in  spectral characteristics across different classes, impacts class performance. 
As a motivating example, in Figure \ref{fig:motivation}, we estimate and examine the spectrum for different classes in a pre-trained ResNet-50 \cite{he2016deep}. We make two important observations:
\begin{itemize} 
    \item First, we find that the eigenvalue distributions differ significantly between top-performing and worst-performing classes, with top-performing classes exhibiting a uniformly \emph{smaller} spectrum (Figure~\ref{fig:motivation}(A)).
    
    \item Second, when comparing 
    eigenvalues of different pretrained encoders,
    we observe that the distribution of the eigenvalues shifts considerably. For example, classes might have similarly ranked eigenvalues in one encoder, but vastly different eigenvalues in another
    (see Figure~\ref{fig:motivation}(B), more in Appendix Figure~\ref{fig:appendix_dist_compare}).
\end{itemize}




\begin{figure}[!t]
    \begin{center}
    \centerline{{\includegraphics[width=0.92\columnwidth]{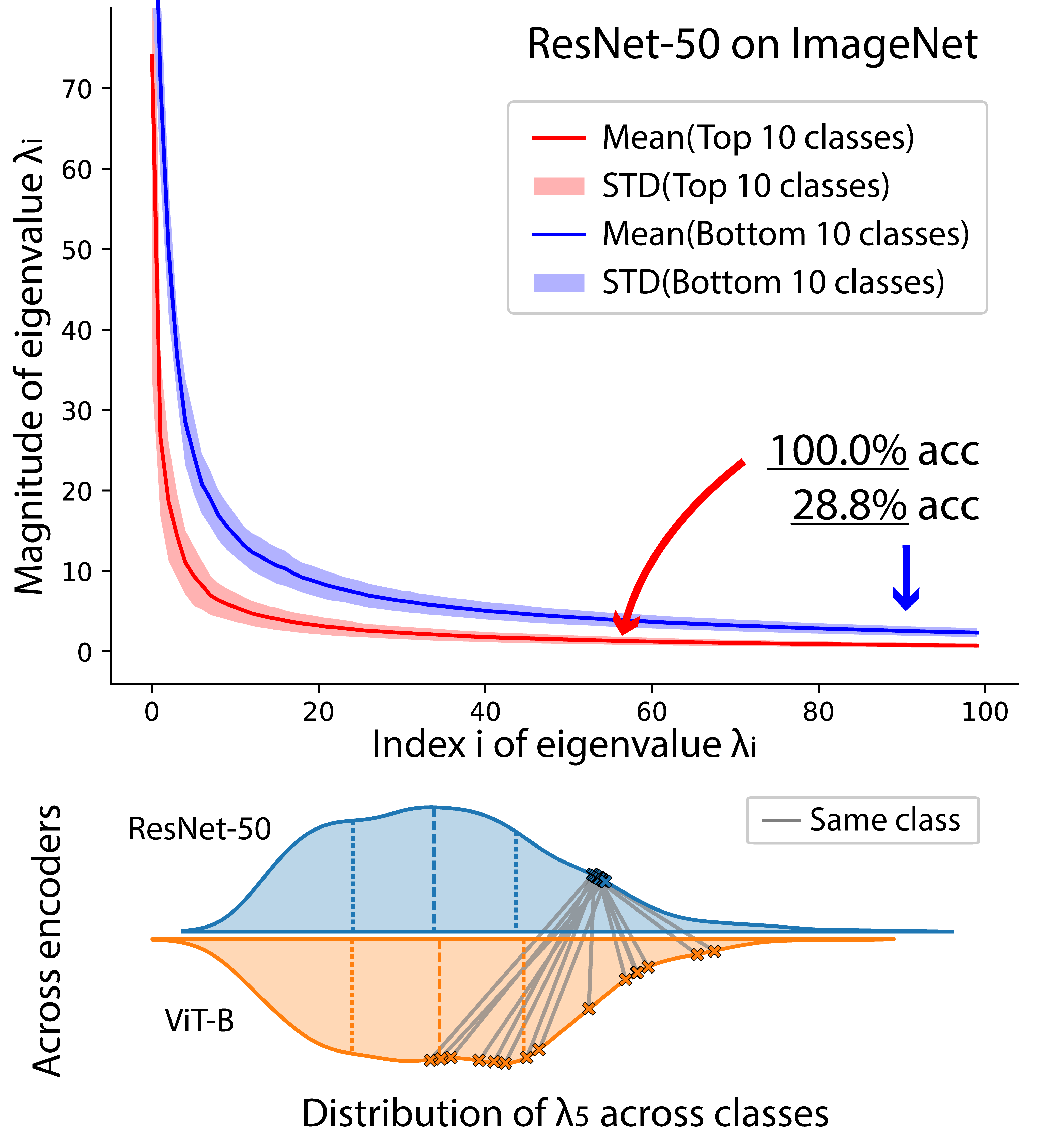} }}%
    \caption{\textit{Spectral imbalance in different pretrained models.} (A) We plot the spectrum of the top and bottom performing classes and find that classes that achieve lower accuracy have larger eigenvalues. (B) Histogram of the $k = 5$ eigenvalue across classes for two encoders (ResNet-50, ViT-B) with the quartiles indicated. If we look at classes with similarly ranked or ``adjacent'' eigenvalues (upper quartile) in ResNet-50, we find that they can be mapped to very different positions in the distribution of eigenvalues in another encoder (ViT-B).
    \label{fig:motivation}}
    \end{center}
    \vspace{-4mm}
\end{figure}

Both of these observations motivate the need to understand the differences in class-dependent spectra, and 
how different sources of spectral imbalance may impact downstream performance.


\vspace{-2mm}
\subsection{Examining the role of the spectrum  in class-dependent generalization}\label{sec:toyex}


To see more concretely how spectral balance across classes can play a role in \emph{class-dependent generalization}, we first consider a simple linear classification setting, where the features follow a 2-class Gaussian mixture model (GMM):
\[
\blx \given y \simiid \mathcal{N}(y \blTS, \blSigma_y).
\]
Here, $\bltheta^* \in \R^p$ is an unknown signal controlling the class means, $y \in \{-1, 1\}$ is the label, and $\blSigma_y$ are the covariance matrices for each class. 
For any fixed estimator $\blthetahat$, inference is performed as $\hat{y} = \text{sign}(\blx^\top \blthetahat)$. Then, the class-dependent probability-of-error (POE) can be written as:


\vspace{-2mm}
\begin{align}
\label{eq:POE}
\text{POE}(\blthetahat \vert y) & \coloneqq \P \braces{\text{sign}(\blx^\top \blthetahat) \neq y | \blx \text{ in class } y } \\ 
& = 
Q\left(\frac{\langle \blthetahat, \blTS \rangle}{\Vert \blSigma_{y}^{1/2} \blthetahat \Vert_{2}}\right)
\end{align}
where $Q(\cdot)$ is the Gaussian Q-function (see proof in Appendix \ref{app:theory}). Fixing the estimator $\blthetahat$, we can see that the class dependency of the POE is captured by the term $\Vert \blSigma_{y}^{1/2} \blthetahat \Vert_{2}$. 
Hence, it is natural to consider differences in this quantity as indicators of how differently 
an estimator
can be expected to perform on the two classes. 
In particular, an estimator $\blthetahat$ will achieve small class gap when $\Vert \blSigma_{y=1}^{1/2} \blthetahat \Vert_{2} \approx \Vert \blSigma_{y=-1}^{1/2} \blthetahat \Vert_{2}$, for which a sufficient (but by no means necessary) condition is that $\blSigma_{y=1} \approx \blSigma_{y=-1}$. 

As a concrete example, suppose for simplicity that the eigendecompositions of the class covariances are $\blSigma_{y} = \sum_{j} \lambda_j^{(y)} \blv_j \blv_j^T$, i.e., both share the same eigenvectors. Then, 
\[
    \textit{ClassGap}(\blthetahat) \coloneqq \abs*{\text{POE}(\blthetahat \vert y=-1) - \text{POE}(\blthetahat \vert y=1)} \\
\]
is positively correlated with the quantity
\[
    \abs*{\sum_{j}(\lambda_j^{(y=-1)} - \lambda_j^{(y=1)}) \langle \blthetahat, \blv_j\rangle^2}
\]


From these expressions, we can see that if the eigenvalues of $\blSigma_{1}$ are uniformly smaller than those of $\blSigma_{-1}$, Class 1 will have smaller generalization error. 
This conforms with the familiar intuition that classes with smaller variance (i.e., less noise) should be ``easier'' to identify, as shown in Figure~\ref{fig:motivation}. Moreover, the effect of an eigenvalue gap between the two classes is unevenly distributed between different coordinates, depending on the structure of the estimator $\blthetahat$. 


Of course, in general, the connection between the spectra of features corresponding to a specific class and the per-class generalization is considerably more nuanced, since $\blthetahat$ is itself learned from training examples and depends crucially on the properties of the spectra of both classes. 
Indeed, many recent works on benign overfitting in linear regression and classification~\cite{bartlett2020benign,muthukumar2020class, cao2021risk, wang2021binary} have shown that the properties of the (overall) data spectrum play a vital role in achieving good overall performance, particularly in the overparameterized regime. 
In the forthcoming sections, we extend these insights to a fundamental understanding of how the \emph{class-wise spectrum} affects \textit{per-class} performance, and build the framework of spectral imbalance.
\section{An exact characterization of the class gap}
\label{sec:theory-formal}


In this section, we explore the factors that govern class imbalances by studying the 
case of high-dimensional linear classification in the GMM setting introduced in Section \ref{sec:theory-intuition}, where the mixtures correspond to the classes. The GMM is a natural way to study class covariance differences and provides insights into the types of representations that might result in larger class biases in practice. Here, we provide an exact asymptotic characterization of the per-class generalization error and class gap. Our framework allows us to model nuanced relationships between the eigenvalues of each class covariance, providing a rigorous testbed for studying their effect on per-class generalization.
\subsection{Problem formulation}
We consider a training scheme where we are given $n$ i.i.d. samples $(\blx, y)$ from the GMM setting in Section \ref{sec:toyex}, where $n_1$ (resp., $n_{-1}$) samples come from Class $1$ (resp., Class -1). We study \emph{empirical risk minimization} (ERM) estimators of the form
\[
\blthetahat \coloneqq \arg \min_{\bltheta \in \R^p} \frac{1}{n} \sum_{i=1}^n \scrL(y_i \ip{\blx_i}{\bltheta}) + r \normt{\bltheta}^2,
\]
where $\scrL$ is a convex loss function and $r>0$ is the ridge regularization parameter. Our goal is to characterize the per-class test error of the resultant $\blthetahat$, given by Equation~\eqref{eq:POE}.

Our theory considers the modern \textit{high-dimensional} setting where both the model complexity and the number of training samples are large. Specifically, we study the per-class error in the asymptotic limit where $n, p \to \infty$ jointly, with the ratios $\frac{n}{p} = \delta$, $\frac{n_y}{n} = \pi_y$ fixed (and clearly, $\pi_1 + \pi_{-1} = 1)$.\footnote{We note here that the quantities $\blthetahat, \blSigma_y, \bltheta^*$, and $\text{POE}(\blthetahat \given y)$ should be interpreted as sequences indexed by $n$; however, we will suppress the dependence on $n$ in the notation for clarity.}
Our theory allows for overparameterization ($\delta < 1$) or underparameterization ($\delta > 1$), and for any constant regularization parameter $r > 0$. 

\begin{figure*}[!t]
    \centering
    \includegraphics[width=\textwidth]{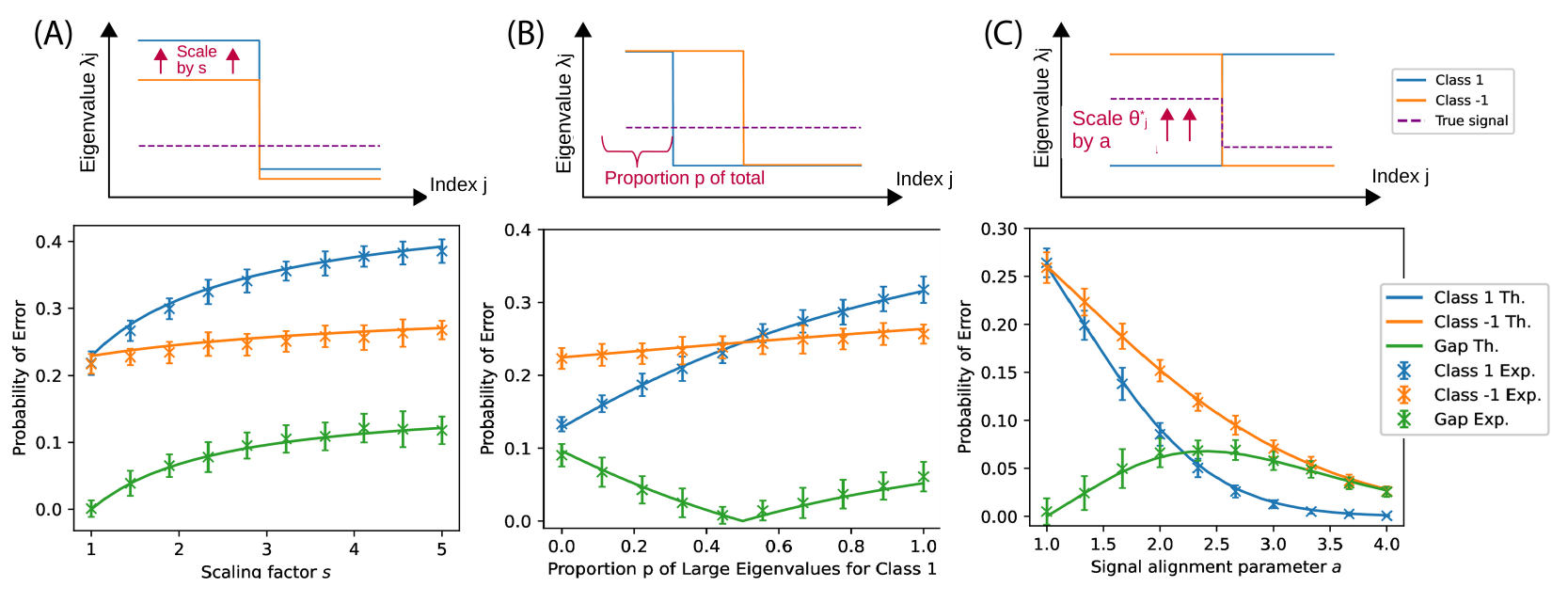}
    \caption{Spectral imbalance in the GMM for the settings in Section \ref{sec:theory-insights}. Top Row: Visualization of the spectra of both classes. Bottom Row: Theoretical predictions of Theorem \ref{thm:main_thm} (solid lines) and numerical simulations (average over 50 trials $\pm 1$ std.) for the per-class error.}%
    \label{fig:theory}
\end{figure*}

\paragraph{Notation.} 
The Mahalanobis norm of a vector, defined for a positive semi-definite matrix $\mathbf{H}$ is given by $\|\blx\|_{\mathbf{H}}=\sqrt{\blx^\top\mathbf{H}\blx}$. 
We also define the following function, which we denote the \textit{$L$th order Moreau envelope function} of a given convex function $f: \R \to \R$:
\[
\scrM^{(L)}_f(x_1, .\hfil.\hfil., x_L ;\tau_1, .\hfil.\hfil., \tau_L) \coloneqq \min_{v \in \R} f(v) + \sum_{\ell=1}^L \frac{1}{2\tau_{\ell}} \parens{x_\ell-v}^2
\]
Note that for $L=1$ this is the classical Moreau envelope function of $f$. We will often use the shorthand notation $\scrM^{(2)}(x_y; \tau_y) = \scrM^{(2)}(x_1, x_{-1}; \tau_1, \tau_{-1})$. The notation $\overset{P}{\to}$ denotes convergence in probability, $\overset{\scrW_2}{\to}$ denotes convergence in Wasserstein-2 distance, and the empirical distribution of a random variable $a_j$ is denoted $\frac{1}{p}\sum_{j=1}^p \delta(a_j)$.



\vspace{-2mm}
\subsection{Main results} 
We now state our general asymptotic result for the GMM with spectral imbalance. Our analysis relies on the Convex Gaussian Min-Max Theorem \cite{thrampoulidis2018precise} and is influenced (particularly in the treatment of anisotropic covariances) by \cite{montanari2019generalization} and \cite{taheri2020asymptotic}, who consider, respectively, max-margin classification and adversarial training in the GMM \emph{without} spectral imbalance. After stating the assumptions and theorem, we briefly outline major differences in the proof technique while deferring the full proof 
to Appendix \ref{app:theory}.
\begin{assumption}\label{assump} The target signal $\bltheta^*$ and class covariances $\blSigma_y$ satisfy the following: 
\begin{enumerate}[topsep=0pt,itemsep=-0.5ex,partopsep=1ex,parsep=1ex]
    \item The $\blSigma_y$ are diagonal.
    \item $\norm{\bltheta^*}_{\blSigma_y^{-1}} \to \zeta_y$ and $\norm{\bltheta^*}_2 \to C$ for some constants $\zeta_y, C>0$.
    \item The empirical joint distribution of $(\sqrt{p} \bltheta^*_j, \lambda_{j}^{(1)}, \lambda_{j}^{(-1)})$ converges to some distribution, denoted $\Pi$:
$\frac{1}{p} \sum_{j=1}^p \delta(\sqrt{p} \bltheta^*_j, \lambda_{j}^{(1)}, \lambda_{j}^{(-1)}) \overset{\scrW_2}{\to} \Pi$. 
\end{enumerate}    
\end{assumption}
Due to rotational invariance, the first assumption is for simplicity and can be relaxed to the setting where the covariances are simultaneously diagonalizable. We also note that the last assumption is quite flexible; as we will see in Section~\ref{sec:theory-insights}, various choices of $\Pi$ can reflect very different types of spectral imbalance between the two classes.
\begin{theorem}\label{thm:main_thm}
    Let $G, H_1, H_{-1} \simiid \scrN(0,1)$ and $(T, L_1, L_{-1}) \sim \Pi$. Under Assumption \ref{assump}, the per-class POE can be written as a function of scalars $\mu_y$, $\alpha_y$, and $\zeta_y$:
    \vspace{-2mm}
    \[
    \text{POE}(\hat{\bltheta} \given y) = Q\left(\frac{\langle \blthetahat, \blTS \rangle}{\Vert \blSigma_{y}^{1/2} \blthetahat \Vert_{2}}\right)
    \overset{P}{\to} Q\parens*{\frac{\mu_y \zeta_y^2}{\sqrt{\alpha_y^2 + \mu_y^2\zeta_y^2}}},
    \]
    where ($\mu_y$, $\alpha_y$) are found in the optimal solution (if it is unique) of the following min-max problem over 12 scalar variables:
    \begin{equation}\label{eq:minmax}
    \begin{aligned}
    &\min_{\substack{\alpha_y, \tau_y \geq 0\\ \mu_y \in \R}} \max_{\substack{\beta_y, \gamma_y \geq 0 \\ \eta_y \in \R}} r\E  \scrM^{(2)}_{(\cdot)^2}\parens*{\frac{\alpha_y \beta_y H_y}{\gamma_y \sqrt{\delta L_y}} + \frac{\eta_y \alpha_y T}{\gamma_y \zeta_y^2 L_y} ; \frac{\alpha_y r}{\gamma_y L_y}} \\
    &+ \sum_{y \in \{-1, 1\}} \left[\pi_y \E_G \scrM^{(1)}_\scrL \parens*{\mu_y\zeta_y^2 + \sqrt{\mu_y^2\zeta_y^2 + \alpha_y^2}G ; \frac{\tau_y}{\beta_y}} \right.\\ &\left.-\frac{\eta_y^2 \alpha_y}{2\gamma_y \zeta_y^2} - \frac{\mu_y^2 \gamma_y\zeta_y^2}{2\alpha_y}\right.
    -\left.\frac{\alpha_y \beta_y^2}{2\delta\gamma_y} + \frac{\beta_y \tau_y}{2} - \frac{\alpha_y \gamma_y}{2} + \eta_y \mu_y \right]\hspace{-1mm}.
    \end{aligned}
    \end{equation}
    

\end{theorem}

First, we note that the min-max optimization specified in Theorem \ref{thm:main_thm} is straightforward to solve by gradient descent/ascent. Hence, this theorem provides a simple analytical framework to assess the \textit{exact} asymptotic impact of various high dimensional and spectrally-imbalanced GMM settings. Specifically, we can apply the above theorem for different choices of the distribution of features and ground truth (i.e., ~$\Pi$) to characterize the precise impact of various types of spectral imbalance. While these will be the main factors which we study below, the theorem also provides the flexibility to analyze the effect of other parameters like the choice of loss function $\scrL$ or regularization strength $r$. While we do not explore this here, the result also allows for comparison between recently proposed alternative performance metrics in class-imbalanced settings, which are functions of the per-class accuracies \cite{mullick2020appropriateness}.

\vspace{-2mm}
\paragraph{A note on the proof technique:} The objective function in Equation~\eqref{eq:minmax} is quite similar to the min-max program derived for the GMM with equal class covariances in \cite{taheri2020asymptotic}; however, there are a few key differences in our proof technique which allow us to generalize that result and obtain a more refined \textit{per-class} characterization of the error. Firstly, 
we use a \textit{decoupling trick} that introduces two variables (one for each class) for each variable in the objective function of \cite{taheri2020asymptotic}. This allows us to isolate the quantities of interest for each class and apply the CGMT separately for each class. Secondly, in contrast to \cite{taheri2020asymptotic}, our result  features the second order Moreau envelope term, which functions as a way to link the decision variables corresponding to each class.

\subsection{Insights on spectral imbalance}\label{sec:theory-insights}
We can use the predictions of Theorem \ref{thm:main_thm} to precisely study the impact of various types of spectral imbalance on per-class performance. In Figure \ref{fig:theory}, we plot the asymptotic predictions of the per-class error and the class gap for three different types of spectral imbalance which can arise in GMMs. To isolate the effect of spectral imbalance, the predictions in this section fix the amount of data from each class as identical (i.e. $\pi_y = 0.5$), the parameterization ratio $\delta = 2$, regularization parameter $r=0.5$ and the loss as the squared hinge loss, $\scrL(t) = \max(0, 1-t)^2$. In each case, we verify the theoretical predictions via numerical simulations averaged over $50$ independent draws of $n=1000$ training examples. The results indicate that the asymptotic predictions accurately describe the per-class behavior even for moderate values of $n$ and $p$. We provide further details of the setup for these simulations in Appendix \ref{app:sim_details}. We also explore the scenario where sample imbalance and spectral imbalance are \textit{both} present in Appendix \ref{app:sampleimbalance}.

\vspace{-2mm}
\paragraph{(A) Impact of eigenvalue scaling:} In this example, we vary the relative scale of some of the eigenvalues of Class $1$ while fixing the target signal as $\sqrt{p}\bltheta^*_j = 1$ for all $j$. Specifically, we fix the eigenvalues of Class -1 to be either $0.5$ or $2$, in equal proportion. Then, we set all the eigenvalues for Class 1 to be a multiplicative scaling $s \geq 1$ of the eigenvalues of Class -1. We can see in Figure \ref{fig:theory}(A) that as $s$ increases, the performance of Class 1 degrades severely while that of Class -1 also degrades, but to a lesser extent. In line with the general intuitions in Section \ref{sec:toyex}, this results in a net increase in the class gap.

\vspace{-2mm}
\paragraph{(B) Impact of eigenvalue decay rate:} Fixing $\sqrt{p}\bltheta^*_j = 1$ for all $i$, we again consider a bi-level covariance model, where the eigenvalues of each class take one of two values: $2$ or $0.5$. We fix the proportion of larger eigenvalues in Class -1 to be $0.5$ while varying the proportion $p \in [0,1]$ of larger eigenvalues for Class $1$.   This example indicates that there can be an important trade-off between overall performance and the class gap: the choice $p=0$ (all small eigenvalues) yields the best per-class performance for both classes while $p=0.5$ minimizes the class gap.

\vspace{-3mm}    
\paragraph{(C) Impact of alignment with the target signal:} The previous notions of spectral imbalance ignored the relative shape of each marginal distribution with respect to the target signal $\bltheta^*$. We now consider a situation where the two class covariances are misaligned, so that either $(\lambda_j^{(1)}, \lambda_j^{(-1)}) = (0.5, 2)$ or $(\lambda_j^{(1)}, \lambda_j^{(-1)}) = (2, 0.5)$.  We then increase the magnitude of $\bltheta^*_j$ by a factor $a$, \textit{only} in the directions $j \in [p]$ where $\lambda_j^{(-1)} = 2$. Here, as we increase the parameter $a\geq1$, the overall signal strength increases, so the global performance improves. However, due to the differences in the covariance alignment, the improvement is not uniform between the two classes, and Class 1, which has smaller variance in the directions of increased signal strength, improves more quickly. Interestingly, for large $a$ the class gap decreases as the effect of large signal strength outweighs the differences between the class spectra.

\section{Spectral imbalance in pretrained features}
\label{sec:experiments}

In this section, we build on our theoretical findings to explore how the framework of spectral imbalance may be used to understand and mitigate class bias in the image representation space of pre-trained models.

\subsection{Experimental setup}


\paragraph{Settings.}~
We consider a common scenario in representation learning, where encoders are used to extract features, and linear classifiers are trained on these features for downstream evaluation.
Formally, a dataset $D=\left\{\left(\mathbf{x}_i, y_i\right)\right\}_{i=1}^n$ of $n$ images is used to train a network consisting of a pretrained encoder $f_\omega(\cdot)$ and a linear classifier $g_\theta(\cdot) = \langle \theta, \cdot\rangle$. The encoder produces a latent representation for each image following $\mathbf{z}_i = f_\omega(\mathbf{x}_i)$. In this work, we aim to understand the properties of the resulting representation space $\mathcal{Z}=\{\mathbf{z}_i\}_{i=1}^n$, and how the (estimated) spectral properties of the distribution on $\mathcal{Z}$ can be correlated to the classification results of the classifier $g_\theta(\cdot)$ for each class. 
Note that the role of the latent features $\{\mathbf{z}_i\}_{i=1}^n$ in representation learning is directly comparable to the role of the original features $\{\mathbf{x}_i\}_{i=1}^n$ in the preceding GMM; this is because in both cases the downstream model is linear.

\vspace{-3mm}
\paragraph{Empirical estimation of the class-dependent spectra.}

To obtain the estimated eigenspectrum, we first 
estimate the empirical class-dependent covariance matrix as:
\[
\blSigma_C\left({f}_\omega\right)=\frac{1}{|\Omega_C|} \sum_{i \in \Omega_C} (\mathbf{z}_i - \overline{\mathbf{z}}_{c}) (\mathbf{z}_i - \overline{\mathbf{z}}_{c})^T,
\]
where $C$ is a class, $\Omega_C$ denotes the set of examples with label $y_i=C$, and $\overline{\mathbf{z}}_c$ denotes the empirical mean of features with class $C$.

The eigenvalue decomposition of the  covariance matrix is then computed as $\blSigma_C = \mathbf{V}_C \mathbf{\Lambda}_C \mathbf{V}_C^{-1}$, where $\mathbf{\Lambda}_C$ is a diagonal matrix with nonnegative entries $\lambda_i^{(C)}$, and the columns of $\mathbf{\mathbf{V}}_C$ correspond to the eigenvectors of $\blSigma_C$. Without
loss of generality, we assume that $\lambda_1^{(C)} \geq \lambda_2^{(C)} \ldots \geq \lambda_m^{(C)}$, where $m$ is the rank of $\blSigma_C$. The resulting set of eigenvalues $\lambda_i^{(C)}$ is the (empirical) eigenspectrum of class $C$.


\paragraph{Pretrained networks.}
In what follows, we evaluate the representation space of a diverse set of classification models to understand their class bias.
Specifically:
\begin{itemize}[noitemsep]
     \item 
    \textbf{Residual networks:} 
    Using residual connections to bridge convolutional blocks, ResNets are a class of deep convolutional neural networks (CNNs) that is instrumental in many visual recognition tasks \cite{he2016deep}. 
    \item 
    \textbf{Improved CNNs:} 
    DenseNet \cite{huang2017densely} and EfficientNet \cite{tan2019efficientnet} are two representative convolutional architectures that aim to improve ResNets.
    They design different strategies to build connections across layers to improve the 
    information flow inside the networks.
    \item 
    \textbf{Vision transformers:} 
    The vision transformer splits an image into patches and uses self-attention to study their interactions \cite{dosovitskiy2020image}. ViT aggregates global information at earlier layers, creating significantly different representations than CNNs \cite{raghu2021vision}.
    \item 
    \textbf{Transformers without self-attention:} Recent studies show that self-attention is not required for obtaining good representations in vision. Along this line, we study MLP-Mixer \cite{tolstikhin2021mlp} and PoolFormer \cite{yu2022metaformer}, which uses multi-layer perceptrons and 
    average pooling to mix information across patches, respectively. 
\end{itemize}

We found all checkpoints of the pretrained models in Torchvision \cite{marcel2010torchvision} and timm \cite{rw2019timm} (see details in Appendix~\ref{app:exp}).
For all experiments, we use the standard ImageNet ILSVRC 2012 dataset \cite{deng2009imagenet}, which contains $C=1000$ object classes with an average of 1281/50 training/validation images per class. We also provide results on smaller datasets in Appendix~\ref{app:smaller_datasets}.

\begin{figure}[t!]
\begin{center}
\centerline{\includegraphics[width=\columnwidth]{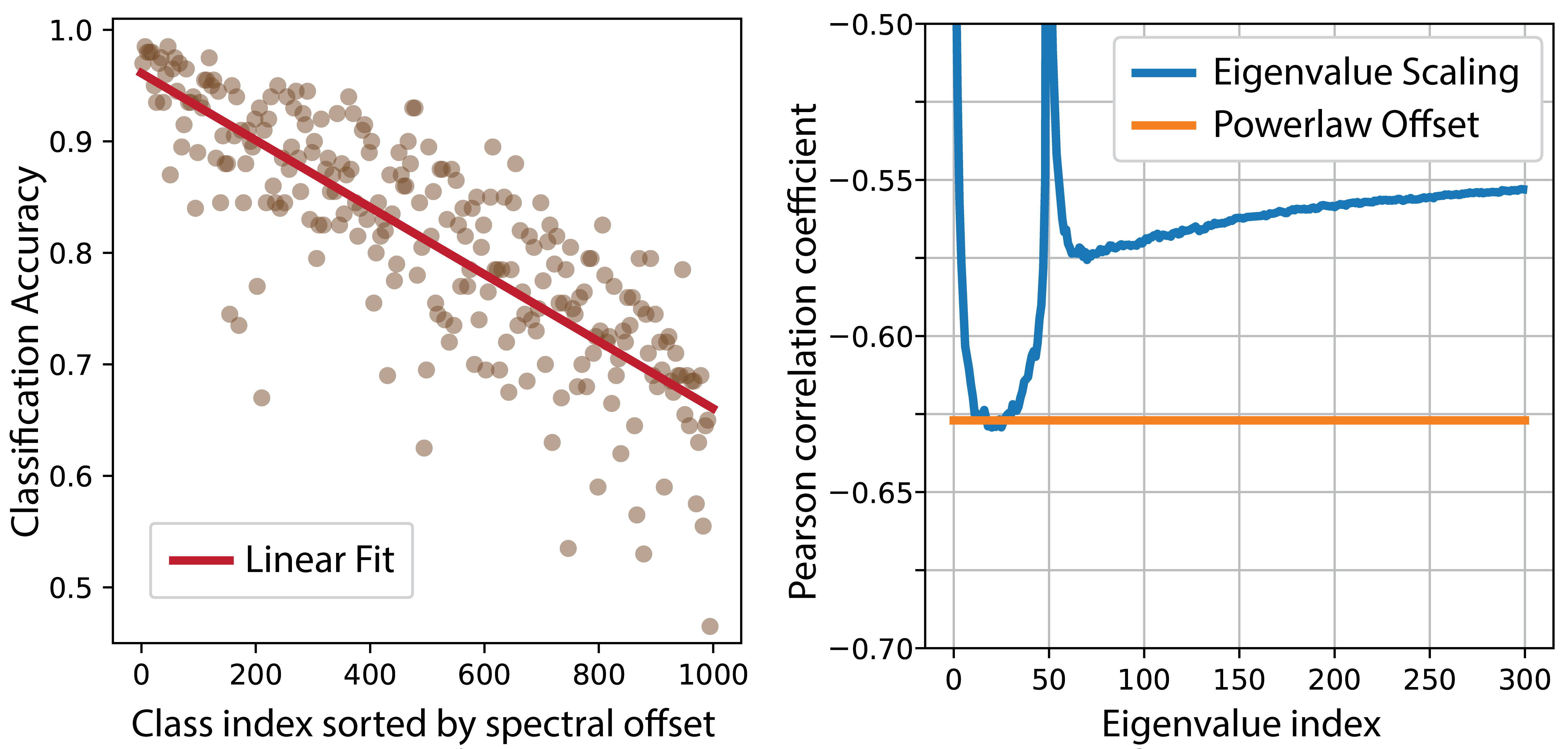}}
\caption{{\em Examining the relationship between class-dependent spectra and performance.} 
(left) The spectral offset of each class vs. their classification accuracy, computed for ResNet-50 on the validation set of ImageNet.
(right) Pearson correlation coefficient between class accuracy and individual eigenvalues (blue) and the power law offset (red).}
\label{fig:correlation}
\end{center}
\vskip -0.2in
\end{figure}

\subsection{Spectra correlate with class-wise performance}\label{sec:si-correlation}

\paragraph{Individual eigenvalues.} We first analyze the correlation between the class-wise performance and the magnitude of individual eigenvalues of each class.
To do this, we compute the Pearson correlation coefficient (PCC) between the set of eigenvalues $\{\lambda_j^{(c)}\}$ (validation set) and the set of across-class accuracies, $\{\text{Accuracy}(c)\}$, for every eigenvalue index $j$. In Figure~\ref{fig:correlation}, 
we observe that the class-wise accuracy shows a strong negative correlation with most individual eigenvalues (right), in line with the theoretical results on eigenvalue scaling in Figure \ref{fig:theory}(A). 
Across the 11 models examined, we obtained an average maximum negative correlation of -0.69. 
Interestingly, unlike the examples in Section \ref{sec:theory-insights}, the first few eigenvalues do not show a strong correlation for most models, potentially illustrating unique properties of the features learned by real-world encoders. 



\paragraph{Power law decay offsets.} Natural images are known to exhibit a power law decay in their eigenvalues \cite{ruderman1994statistics}. 
Recent work has shown that this also holds in visual representations in pretrained models \cite{ghosh2022investigating} and can be used to examine the quality of learned representations. 
Thus, we fit the class conditional spectrum to $\lambda_i = a i^{-b}$, and observe how the parameters $(a, b)$ indicate the class-wise accuracy. Interestingly, we observed that the scale $a$ (the \textit{offset} on a log-scale) correlates most strongly with class accuracy (an average of -0.68 across 11 models), while the power (i.e. rate of decay) $b$ does not strongly correlate with class-wise performance. 
This phenomenon demonstrates that there is a nuanced relationship between the specific distribution and decay rate of eigenvalues (cf. Figure \ref{fig:theory}(B)) and the class performance. We provide further details and also measure the correlation between other decay rate metrics and class-wise performance in Appendix \ref{app:exp_pcc_score}.

\subsection{Spectral imbalance 
in your pretrained model}\label{sec:si-encoder}

Next, we conduct a systematic investigation of the per-class accuracy discrepancies across encoders and show that spectral imbalance is often a measurable property of the encoder.
Surprisingly, out of the 55 unique pairs of encoders that we consider, an average of 31\% classes had a ranking (with duplicates) difference over 100 out of $C=1000$ total classes, with the input data held constant.
To better understand the large variability of per-class accuracy across encoders, we propose the \textit{Spectral Quantile Score} (SQS), which measures the spectral imbalance between classes in a given encoder.

{\bf Spectral Quantile Score.}~
Writing the set of per-class accuracies 
as a sorted list of numbers $\{{Acc}_1, \dots, {Acc}_C \}$ in descending order, we define the \textit{empirical} class bias score of an encoder as:
\[
 \textit{Class bias} = \left(\sum_{k=1}^{L} {Acc}_k - \sum_{k=C-L+1}^{C} {Acc}_k \right)/ \left(L \times \overline{Acc} \right),
\]
where $C$ is the total number of classes, $\overline{Acc}$ is the average accuracy over all classes, and $L$ is a selected cutoff. In our experiments on ImageNet,  we set $L=100$ throughout.

To compute the imbalance across classes in their spectra, we can similarly sort our per-class spectral metric of interest (i.e., spectral offset, specific eigenvalues) as $\{s_1, \dots, s_C \}$ and compute the
the Spectral Quantile Score (SQS):
\begin{equation*}
 \text{SQS} = s_{L}/s_{1}. 
\end{equation*}
Again, we set $L=100$ for our experiments.
As shown in Figure \ref{fig:emperical_main}(A), we found the SQS of different encoders is indicative of their empirical class bias (0.90 PCC), where latent spaces with a higher SQS would likely give higher empirical class bias.
The strong correlation between SQS and empirical class bias emphasizes the importance of investigating the spectra imbalance in encoders, which accurately reflects and evaluates the encoders' empirical class bias.
Moreover, as the worst class performance is often strongly associated with the final average performance of an encoder \cite{arjovsky2022throwing}, understanding spectral imbalance provides an additional lens into the representation space without training actual classifiers.


\paragraph{Testing for class disparity in a new encoder.} Next, we investigate how class-dependent spectra can be used to estimate the performance of a new encoder. 
Following \cite{balestriero2022effects}, we used a one-sided Welch’s t-test \cite{welch1947generalization} to validate our hypothesis that the per-class accuracy is significantly lower when their training set spectra have a higher offset parameter $a$.
We define the random variable as the accuracy difference between two encoders $\Delta_{f_1, f_2}(c) \coloneqq \operatorname{Acc}_{f_1}(c) - \operatorname{Acc}_{f_2}(c)$, and  $\Delta_{-}$ represents the set of classes with spectral offset difference $a_{f_1}(c) - a_{f_2}(c) < 0$ being negative. We define the null hypothesis as $H_0 = \mathbb{E}[\Delta] > \mathbb{E}\left[\Delta_{-}\right]$.
Intuitively, rejecting the hypothesis means the classes with a higher spectral property would give lower class-wise accuracies.
We obtain that there is enough evidence to reject this hypothesis with 95\% confidence for 98 out of 110 pairs of encoders; and 99\% confidence for 80 out of 110 pairs of encoders.
This demonstrates that spectral properties might reliably predict class-wise accuracies for new encoders, despite the noise in spectrum estimation.

\begin{figure*}[!t]
    \centering
    \includegraphics[width=0.97\textwidth]{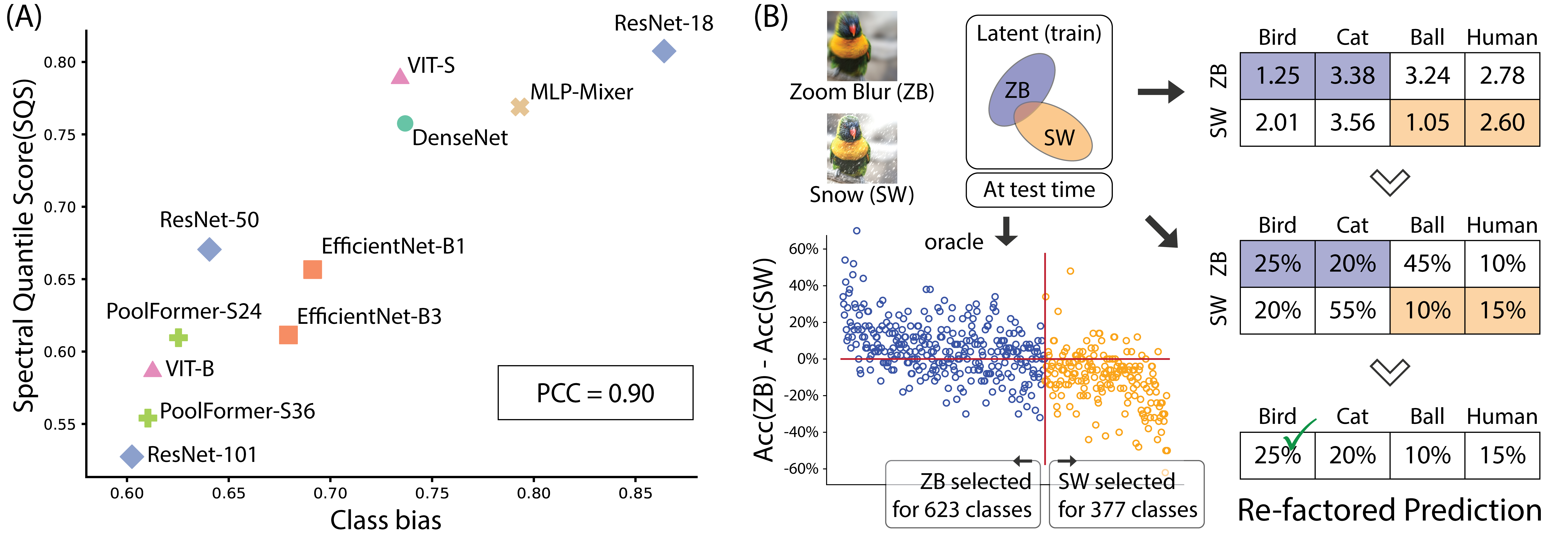}
    \caption{{\em Selecting model and data with spectra.} (A) Given different encoders, the Spectral Quantile Score (SQS) is predictive of the empirical class bias of the representation spaces. Note that SQS is estimated on training set without deploying the classifier(s). 
    (B) Given different data augmentation plans, we can use the class-level spectra to assign data to their optimal augmentation plan in a class-dependent way. This assignment can be used at test time to re-factor the prediction and boost performance across all classes without any re-training or modification to the encoder.}
    \label{fig:emperical_main}
\end{figure*}


\subsection{Combine augmentations using spectral information}
\label{exp:data_augmentation}

Recent work identified that commonly used data augmentations can in fact worsen class gap~\cite{balestriero2022effects}. Thus, we sought to study the effect of \emph{data augmentation} on class disparity and ask whether different augmentations can be used to alleviate  gaps in class accuracies.





\paragraph{Do different augmentations impact spectral imbalance?}~ 
To begin, we supplement the results of Section~\ref{sec:si-encoder} by showing that spectral imbalance correlates strongly with class gap not only across encoders, but also across data transformations.
Following \cite{hendrycks2019benchmarking}, we generated 16 versions of ImageNet where the images are transformed differently. 
Specifically, given a sample ${\bf x}$ from our train/test dataset $\mathcal{D}$, we start by transforming every image with the 16 candidate augmentations as $\mathbf{x}^{\ell} = t_{\ell}(\mathbf{x} )$, where $t_{\ell} \in \mathcal{T}_{\ell}$ is a sampled augmentation operation from the $\ell^{\rm th}$ augmentation plan. After feeding these transformed images through our pretrained encoder $f_{\omega}$, we obtain 16 different sets of latents $\{ {\bf z}_{\ell} = f_{\omega}(\mathbf{x}^{\ell}) \}_{\ell=1}^{16}$.
Then, our prediction corresponding to the $\ell^{th}$ augmented test sample is given by the $C$-dimensional logit output $\mathbf{L}_{\ell} = g_{\theta}({\bf z}_{\ell})$.

As in Section~\ref{sec:si-correlation}, we fit a power law to all classes and use the offset parameter $a$ as the studied spectral property.
Between the training set spectral properties and the testing set class bias, we compute the SQS score across all 16 types of augmentations and obtain a PCC of $0.90$, demonstrating strong positive correlation.
Within each class, the spectral property is also
statistically significant for $78.9\%$ out of $1000$ classes. See Appendix~\ref{app:imagenetc_details} for more details.

\paragraph{An ensembling method to combine augmentations using spectral information.}~
Based on the observed correlation between imbalance of the spectral offset and performance for each class, we propose a \textit{spectral re-factoring} method that adaptively combines the predictions across different augmented views at \emph{test time} to improve performance.
To do this, we extract the spectral offset parameter $a$ for each class ($C=1000$) across all candidate augmentations ($K=16$) on the training dataset. This results in a matrix $\mathbf{M} \in \mathbb{R}^{K \times C}$, where each entry $M_{\ell,c} = a_{\ell,c}$ denotes the offset parameter for all latents $\{ {\bf z}_{\ell} \}$ in class $c$.

Consider a test sample $\mathbf{x}$.
Suppose that an `oracle' could give us its true class $c^*$.
Then, we would use the prediction corresponding to the augmented view that achieves minimum `spectral score' of class $c^*$, i.e. use the augmentation $\ell^*$ that minimizes $M_{\ell,c^*}$.
We depict this idea in Figure~\ref{fig:emperical_main}(B) with two candidates Zoom Blur (ZB) and Snow (SW).
In this case, ZB achieves a smaller spectral score for 623/1000 classes, while SW acheives a smaller spectral score for the other 377/1000 classes.
Thus, our oracle would select
the predictions of ZB for the 623 classes with smaller spectra and SW for the next 377 classes. This approach would yield significant improvement across all classes: $6.0\%$ on 623 classes (blue) and $10.2\%$ on 377 classes (yellow).

Of course, in reality we do not have access to the class information $c^*$ to apply the `correct augmentation'. Instead, we propose an ensembling method for assigning final class probabilities based upon the spectral property matrix ${\bf M}$:
\begin{itemize}[leftmargin=*,label={}]

\item {\bf(Step 1)}\hspace{3mm}Pass  $K$ candidate augmentations of each sample through $f$ and $g$ to produce a multi-view logits matrix $\mathbf{L} \in \mathbb{R}^{K \times C}$, where $\mathbf{L}_{\ell}$ corresponds to the logits generated by prediction on augmented data $\mathbf{x}^{\ell}$.

\item {\bf(Step 2)}\hspace{3mm} Create a new logit vector  $\hat{{\bf l}}_c = {\bf L}_{\ell^*(c), c}$ where $\ell^*(c)$ \emph{minimizes} $M_{\ell,c}$. Intuitively, we select the logit corresponding to the `optimal' augmentation for class $c$, based upon the spectral scores in $\bf M$. 

\item {\bf(Step 3)}\hspace{3mm} Classify test data $\mathbf{x}$ as $c(\mathbf{x}) = {\arg \max}_{c \in [C]} \hat{L}_c$.
\end{itemize}

This re-factored prediction is depicted in Figure~\ref{fig:emperical_main}(B).
Essentially, the procedure first \emph{estimates} the correct augmentation from the training set for each class to use in steps 1 and 2, selects the corresponding logit for that class, and uses these adaptively selected logits in a standard classification procedure in step 3.
This ensembling procedure improves Zoom Blur ($65.01\%$) and Snow ($65.12\%$) to a combined accuracy of $67.616\%$,
giving an average performance improvement of $2.55\%$ and an average class gap improvement of $3.52\%$  \textit{without any re-training required}.

Using the same approach, we repeat the experiments on all combinations of augmentations provided by \cite{hendrycks2019benchmarking} in Appendix~\ref{app:imagenetc_details}. We show that, across 15x16 = 240 runs, our proposed ensemble method can stably improve the average performance across two augmentations and the worse augmentation by 2.6\% and 5.9\% respectively.

\section{Related work}

{\bf Class-dependent generalization and class bias.}
Recently,~\cite{balestriero2022effects,kirichenko2023understanding} put forward the interesting observation that data augmentations can create class bias, and that the issue worsens the stronger the augmentation(s) is.
In contrast to our work, these works focus exclusively on the impact of data augmentation and do not identify or characterize the role of spectra in class bias.
They propose a different hypothesis that is less quantifiable; in particular, that classes that are inherently more `fine-grained', or harder to distinguish, are most affected.
Comparatively, our proposed measures of spectral imbalance are precise, quantitative, and demonstrated to correlate with class bias through formal statistical analysis.



{\bf The role of spectra in generalization.} Many recent works have highlighted  the importance of the data spectrum in generalization \cite{bartlett2020benign,muthukumar2020class, lin2022good}, especially in the overparameterized regime. 
Specifically, the works~\cite{chatterji2021finite, wang2021binary, cao2021risk, wang2021benign} characterize the behavior of overparameterized linear classifiers in GMMs through non-asymptotic, finite-sample bounds and show good generalization under favorable spectral properties of the data. However, they all assume equal class covariances and do not explicitly consider the effects of spectral imbalance.
Another line of work provides exact asymptotic bounds on the classification error achieved by linear classifiers in the GMM \cite{deng2022model, mai2019large, loureiro2021learning,kini2021label,taheri2020asymptotic,javanmard2022precise}. Of these, only \cite{loureiro2021learning} allow for different covariances between classes. Their work studies a more general multiclass setting, but does not systematically study per-class generalization. By contrast, we focus on a simpler binary classification setting and use an alternative proof technique based on Gaussian comparison inequalities to obtain a much simpler expression that is easier to numerically evaluate, and thus allows us to model various spectral imbalances.
Finally, our empirical results estimate the eigenspectra either on validation/test data, on augmented training data, or on training data but non-zero training error.
This circumvents ``neural collapse" on training data~\cite{papyan2020prevalence} which would yield misleading estimates of eigenspectra.


{\bf Evaluating representations.}~Assessing the quality of representations is a critical aspect of representation learning. 
The importance stems from its ability to bypass the need for training classifiers, especially when the quality and quantity of downstream data are lacking \cite{martin2021predicting,liu2022seeing,liu2023frequency}.
Representations can be studied from the perspective of manifold analysis~\cite{hauser2017principles}; Previous works show that estimated geometric properties like intrinsic dimension correlate with performance \cite{ansuini2019intrinsic,cohen2020separability,doimo2020hierarchical,valeriani2023geometry}.
Recently, \cite{ghosh2022investigating, agrawal2022alpha} uses the power law decay of eigenspectrum to measure the quality of representations, 
and follow-up works show that
spectral properties can also be used to improve self-supervised learning \cite{he2022exploring,zhang2023spectral,weng2023modulate}. 
In contrast to these works, we especially evaluate the spectral properties of the  representation space from the perspective of class bias and class-dependent generalization error. 

The recent works \cite{ma2022delving, ma2023curvature} also provide a geometric interpretation of the class covariances of learned representations and connect these differences in geometry to differences in performance, both for long-tailed and balanced datasets. Our work provides an alternate and complementary perspective, studying the decay rate and relative magnitude of eigenvalues, rather than the ``volume'', which is related to the product of eigenvalues \cite{ma2022delving} or the Gauss curvature \cite{ma2023curvature} of the learned features. Compared to these works, we provide an explicit theoretical framework for studying spectral imbalance in high-dimensional GMMs, and we focus our attention on pretrained representations and sample-balanced datasets. We compare our empirical findings to \cite{ma2022delving} in more detail in Appendix \ref{app:compare_volume}.
\section{Discussion}

In this work, we introduced the concept of spectral imbalance as a way to characterize class-dependent bias. We studied spectral imbalance in theory and in practice, and provided a new framework for studying the relationship between class and spectral imbalance.
While we identify spectral imbalance as \emph{one of} the factors influencing class disparity, there could be others which we have as yet not identified.
Moreover, our theory assumes linear models 
and our experiments assume pre-trained encoders.
In the future, we would like to understand how feature learning in neural networks could lead to spectral imbalance in the first place. We hope such understanding will help us effectively
mitigate the class bias issue, either through re-training or post-hoc manipulation of pre-trained features.

\section*{Acknowledgements}

We are grateful to the anonymous ICML reviewers for valuable feedback.
This project is supported by an NSF Graduate Research Fellowship (DGE-2039655),
NSF award (IIS-2039741), NSF award (CIF:RI:2212182), NSF
CAREER awards (IIS-2146072, CCF-223915), as well as generous gifts from the Alfred Sloan Foundation, the McKnight Foundation, the CIFAR Azrieli Global Scholars Program, Amazon Research and Adobe Research.

\section*{Impact Statement}
Our work provides new insights into how to mitigate class biases by measuring the spectral properties of representation space. The presented results suggest that large-scale pre-trained models may possess inherent biases in their decision-making processes. Our method provides researchers with tools to better understand and address these issues.
While this work focuses on evaluating and mitigating the biases in discriminative tasks, extending this approach to generative models and tasks can have major societal benefits regarding fairness and auditing of pre-trained models. 

\bibliography{fairDA}
\bibliographystyle{icml2024}

\newpage
\appendix
\onecolumn

\section*{Appendix}
\setcounter{figure}{0}

\section{Theoretical Results for the Gaussian Mixture Model}\label{app:theory}

\subsection{Proofs}
This section contains the proofs for spectrally-imbalanced GMM setting in Section \ref{sec:toyex}. 
\begin{lemma}[Expression for the POE in the imbalanced GMM]\label{lem:poe}
In the GMM setting of Section \ref{sec:toyex}, given an estimator $\blthetahat \in \R^p$, the per-class error and class gap satisfy
\begin{align*}
\text{POE}(\blthetahat \vert y) &\coloneqq \P \braces{\text{sign}(\blx^\top \blthetahat) \neq y | \blx \text{ in class } y } = 
Q\left(\frac{\langle \blthetahat, \blTS \rangle}{\Vert \blSigma_{y}^{1/2} \blthetahat \Vert_{2}}\right)\\
\text{ClassGap}(\blthetahat) &= \abs*{Q\left(\frac{\langle \blthetahat, \blTS \rangle}{\Vert \blSigma_{1}^{1/2} \blthetahat \Vert_{2}}\right) - Q\left(\frac{\langle \blthetahat, \blTS \rangle}{\Vert \blSigma_{-1}^{1/2} \blthetahat \Vert_{2}}\right)}\\
&\leq \frac{1}{\sqrt{2\pi}} \exp\parens*{- \frac{1}{2} \min_y  \frac{\langle \blthetatl, \bltheta^* \rangle^2}{\normt{\blSigma_y^{1/2} \blthetahat}^2}}\abs{\langle \blthetahat, \blTS \rangle} \abs*{\normt{\blSigma_{y}^{1/2} \blthetahat}^{-1} -\normt{\blSigma_{y}^{1/2} \blthetahat }^{-1}},
\end{align*}
where $Q(\cdot)$ is the Gaussian Q-function.
\end{lemma}
\begin{proof} Note we can write $\blx$ as $y \bltheta^* + \blSigma_y^{1/2} \blz$, for $\blz \sim \scrN(0, \blI)$. Then,
    \begin{align*}
        \P \braces{\text{sign}(\blx^\top \blthetahat) \neq y | \blx \text{ in class } y } &= \P \braces{y(\blx^\top \blthetahat) < 0 | \blx \text{ in class } y }\\
        &= \P \braces{y((y \bltheta^* + \blSigma_y^{1/2} \blz)^\top \blthetahat) < 0}\\
        &= \P \braces{\langle \blthetahat, \bltheta^* \rangle  + y\blthetahat^\top \blSigma_y^{1/2}\blz  < 0}\\
        &= \P \braces{\langle \blthetahat, \bltheta^* \rangle  + \norm{\blSigma_y^{1/2}\blthetahat}_2 z  < 0}\\
        &= \P \braces*{z < -\frac{\langle \blthetahat, \bltheta^* \rangle}{\norm{\blSigma_y^{1/2}\blthetahat}_2}}\\
         &= Q\left(\frac{\langle \blthetahat, \blTS \rangle}{\Vert \blSigma_{y}^{1/2} \blthetahat \Vert_{2}}\right),
    \end{align*}
    where, in the fourth line, we use the fact that $y\blthetahat^\top \blSigma_y^{1/2}\blz \overset{d}{=} \norm{\blSigma_y^{1/2}\blthetahat}_2 z,$ for a scalar $z \sim \scrN(0,1).$ The upper bound on the class gap follows from the mean value theorem, after noting that the derivative of the Q-function is the negative of the Gaussian pdf.
\end{proof}

Before proceeding with the full proof of our main technical result, we first provide a brief sketch of the proof technique and the primary differences from previous works which study the balanced GMM setting. 

\paragraph{Proof outline and remarks} Our proof generalizes the result of \cite{taheri2020asymptotic} to the case where each class has a different covariance and allows us to characterize the error in a per-class way. To make the connections to this analysis (and the key differences) clear, we make an attempt to follow the same notation when possible. Due to the covariate imbalance between classes, our analysis of the ERM objective function differs in a few important respects. Typically, application of the CGMT to anisotropic covariates requires a ``whitening'' step to obtain standard normal variables. The key step in our proof is to \textit{separately} whiten the data from each class, which ``decouples'' the objective function into a portion corresponding to each class with its own decision variables. These two portions are tied together via the introduction of a new constraint. 

We then invoke the CGMT once for each class (or more precisely, via a multivariate extension of \cite{dhifallah2021inherent}) to write the problem as an equivalent ``Auxiliary Optimization'' (AO) objective which has a similarly decoupled form. Each of these portions of the AO can be simplified separately using the same ``scalarization'' techniques as in \cite{taheri2020asymptotic}. The additional constraint we introduced during the whitening step then plays a role in the final step of the proof, yielding the second-order Moreau envelope function which depends on the decision variables corresponding to both classes. Finally, we note that, unlike \cite{taheri2020asymptotic}, we do not consider the effect of adversarial training; however, the same proof technique can be extended in a natural way to deal with this case. 

We now proceed to the full proof of Theorem \ref{thm:main_thm}.

\begin{proof}[Proof of Theorem \ref{thm:main_thm}]
We first analyze the ERM optimization through a series of simplifications. 
\begin{align*}
    \min_{\bltheta \in \R^p} \frac{1}{n} \sum_{i=1}^n \scrL(y_i \ip{\blx_i}{\bltheta}) + r \normt{\bltheta}^2 &= \min_{\bltheta \in \R^p} \frac{1}{n} \sum_{i=1}^n \scrL(y_i \ip{y_i\bltheta^* + \blSigma_{y_i}^{1/2}\blz_i}{\bltheta}) + r \normt{\bltheta}^2\\
    &\overset{d}{=} \min_{\bltheta \in \R^p} \frac{1}{n} \sum_{i=1}^n \scrL(\ip{\bltheta^* + \blSigma_{y_i}^{1/2}\blz_i}{\bltheta}) + r \normt{\bltheta}^2,
\end{align*}
where the last line follows from $y_i\blz_i \overset{d}{=} \blz_i$. Then, we re-index the data points as $\blz_{y,i}$ for $y \in \{-1,1\}$ and $i \in [n_y]$ to get
\begin{align*}
    &\min_{\bltheta \in \R^p} \frac{1}{n} \sum_{y \in\{-1,1\}} \sum_{i=1}^{n_y} \scrL(\ip{\bltheta^* + \blSigma_{y}^{1/2}\blz_{y,i}}{\bltheta}) + r \normt{\bltheta}^2\\
\end{align*}
Next, let $\blthetatl_y \coloneqq \blSigma_y^{1/2} \bltheta$ and $\blTStl_y \coloneqq \blSigma_y^{-1/2}\blTS$. We now use a change of variable to write the problem as an optimization over the two variables $\blthetatl_1$ and $\blthetatl_{-1}$ while enforcing that both correspond to the same $\bltheta$.
\begin{align*}
    &\min_{\substack{\blthetatl_y \in \R^p \\ \text{s.t.} \blSigma_1^{-1/2}\blthetatl_1 = \blSigma_{-1}^{-1/2}\blthetatl_{-1}}}  r \normt{\blSigma_1^{-1/2}\blthetatl_1}^2 + \frac{1}{n} \sum_{y \in\{-1,1\}} \sum_{i=1}^{n_y} \scrL(\ip{\blthetatl_y}{\blTStl_y} + \ip{\blz_{y,i}}{\blthetatl_y})\\
\end{align*}
By duality, we can now write this as a min-max problem.
\begin{align*}
    &=\min_{\blthetatl_y \in \R^p }  \max_{\blw \in \R^p} \quad r \normt{\blSigma_1^{-1/2}\blthetatl_1}^2 + \ip{\blw}{\blSigma_1^{-1/2}\blthetatl_1 - \blSigma_{-1}^{-1/2}\blthetatl_{-1}} + \frac{1}{n} \sum_{y} \sum_{i=1}^{n_y} \scrL(\ip{\blthetatl_y}{\blTStl_y} + \ip{\blz_{y,i}}{\blthetatl_y})\\
    &= 
    \min_{\blthetatl_y \in \R^p }  \max_{\blw \in \R^p} \quad r \normt{\blSigma_1^{-1/2}\blthetatl_1}^2 
    + \ip{\blw}{\blSigma_1^{-1/2}\blthetatl_1 - \blSigma_{-1}^{-1/2}\blthetatl_{-1}}\\  &\hspace{8em} + \frac{1}{n}\sum_{y} \brackets*{ \sum_{i=1}^{n_y} \scrL(\ip{\blthetatl_y}{\blTStl_y} + \ip{\blz_{y,i}}{\blTheta_y \blthetatl_y} +  \ip{\blz_{y,i}}{\blTheta_y^\perp \blthetatl_y})}
\end{align*}
where we define $\blTheta_y \coloneqq \frac{\blTStl_y \blthetatl^{*\top}_y}{\normt{\blTStl_y}^2}$ as the orthogonal projection matrix onto $\blTStl_y$. Now, introduce the constraints $v_{y,i} = \ip{\blthetatl_y}{\blTStl_y} + \ip{\blz_{y,i}}{\blTheta_y \blthetatl_y} +  \ip{\blz_{y,i}}{\blTheta_y^\perp \blthetatl_y}$, corresponding to dual variables $\blu_y$ to get:
\begin{equation}
    \begin{aligned}
        &\min_{\substack{\blthetatl_y \in \R^p \\ \blv_y \in \R^{n_y}}}  \max_{\substack{\blw \in \R^p \\ \blu_y \in \R^{n_y}}} \quad r \normt{\blSigma_1^{-1/2}\blthetatl_1}^2 + \ip{\blw}{\blSigma_1^{-1/2}\blthetatl_1 - \blSigma_{-1}^{-1/2}\blthetatl_{-1}}\\  &\hspace{4em} +  \frac{1}{n}\sum_{y \in \{-1,1\}} \brackets*{\ones^\top\scrL(\blv_y) + \ip{\blu_y}{\ones}\ip{\blTStl_y}{\blthetatl_y} + \ip{\blu_y}{\blZ_y \blTheta_y \blthetatl_y} + \ip{\blu_y}{\blZ_y \blTheta_y^\perp \blthetatl_y} - \ip{\blu_y}{\blv_y}}
    \end{aligned}
\end{equation}
Here, $\scrL(\blv_y) \in \R^{n_y}$ is a vector containing $\scrL(\blv_{y,i})$ and $\blZ_y \in \R^{n_y \times p}$ is a matrix with the $\blz_{y,i}$ as rows. We call this problem the Primary Optimization (PO). We note here that for each $y$, $\blZ_y \blTheta_y \perp \blZ_y \blTheta_y^\perp$. Furthermore, $\blZ_1 \perp \blZ_{-1}$. So, for each $y$, we can replace $\blZ_y \blTheta_y^\perp$ by $\hat{\blZ}_y \blTheta_y^\perp$ for an independent Gaussian matrix $\hat{\blZ}_y$. Conditioning on the $\blZ_y$, we can see that the objective function can be written in the form 
\[
\sum_{y \in \{-1,1\}} \ip{\blu_y}{\hat{\blZ}_y \blTheta_y^\perp \blthetatl_y}  + \psi((\blthetatl_y, \blv_y), (\blw, \blu_y))
\]
for a convex-concave function $\psi$. Here, we can apply Theorem 3 of \cite{dhifallah2021inherent}, which formalizes the repeated application of the CGMT in this setting and allows us to write a corresponding Auxiliary Optimization (AO) problem. It is important to note that using this theorem requires that the optimization be over compact sets. This can be relaxed as in \cite{thrampoulidis2018precise} by arguing that the optimal solution in the limit is bounded so we can equivalently consider the PO to be over a sufficiently large bounded set without changing its solution.

This theorem states that if the optimal solution of the AO converges in probability to some value $\phi$, then the optimal solution of the PO converges to the same value. Moreover, if we restrict the AO decision variables to lie outside of some set $\scrS$ and the resultant problem converges to some $\phi' > \phi$, we can conclude that the optimal minimizers/maximizers of the PO belong to $\scrS$ with probability tending to $1$. 

Applying the CGMT, we arrive at the following AO:

\begin{equation}
    \begin{aligned}
        &\min_{\substack{\blthetatl_y \in \R^p \\ \blv_y \in \R^{n_y}}}  \max_{\substack{\blw \in \R^p \\ \blu_y \in \R^{n_y}}} \quad r \normt{\blSigma_1^{-1/2}\blthetatl_1}^2 + \ip{\blw}{\blSigma_1^{-1/2}\blthetatl_1 - \blSigma_{-1}^{-1/2}\blthetatl_{-1}}\\
        &\hspace{4em} +  \frac{1}{n}\sum_{y \in \{-1,1\}} \left[\ones^\top\scrL(\blv_y) + \ip{\blu_y}{\ones}\ip{\blTStl_y}{\blthetatl_y} + \ip{\blu_y}{\blZ_y \blTheta_y \blthetatl_y} + \normt{\blTheta^\perp_y \blthetatl_y}\blg_y^\top \blu_y \right.\\
        &\hspace{8em}+ \left.\normt{\blu_y}\blh_y^\top \blTheta_y^\perp \blthetatl_y  - \ip{\blu_y}{\blv_y}\right],
    \end{aligned}
\end{equation}
where $\blg_y \sim \scrN(0, \blI_{n_y})$ and $\blh_y \sim \scrN(0, \blI_p)$ are independent random vectors. We will now perform a series of simplifications that show that this min-max problem converges in probability to the solution of the scalar problem provided in Theorem \ref{thm:main_thm}. 

First, observe that the terms inside the summation are identical to the terms which appear in the AO which is studied in \cite{taheri2020asymptotic} (ignoring the terms in their analysis which correspond to adversarial training). So, we can ``scalarize'' the AO problem by applying the same analysis separately over each $y$. Due to the similarity of the steps of this part of the proof, we omit some of the details and emphasize places where the simplification of the AO differs. 

To begin, we can decouple the optimization over $\blu_y$ into the optimization over its norm and its unit direction and explicitly solve the maximization over its direction. Hence, we are left only with the maximization over its norm, which we denote through the variables $\beta_y = \normt{\blu_y}/\sqrt{n}$. As in \cite{taheri2020asymptotic}, we introduce new variables $\blrhotl_y = \blthetatl_y$ (enforced via the dual variables $\bllambda_y$) to allow us to separate the terms inside the summation from the terms outside the summation:
\begin{equation}
    \begin{aligned}
        &\min_{\substack{\blrhotl_y, \blthetatl_y \in \R^p \\ \blv_y \in \R^{n_y}}}  \max_{\substack{\blw, \bllambda_y \in \R^p \\ \beta_y \in \R_+}} \quad r \normt{\blSigma_1^{-1/2}\blrhotl_1}^2 + \ip{\blw}{\blSigma_1^{-1/2}\blrhotl_1 - \blSigma_{-1}^{-1/2}\blrhotl_{-1}}\\
        &\hspace{4em} +  \sum_{y \in \{-1,1\}} \left[\frac{1}{n}\ones^\top\scrL(\blv_y) + \frac{\beta_y}{\sqrt{n}}\normt*{-\blv_y + \ip{\blTStl_y}{\blthetatl_y}\ones + \blZ_y\blTheta_y\blthetatl_y + \blg_y \normt{\blTheta_y^\perp \blthetatl_y}}\right.\\
        &\hspace{8em} \left. + \frac{\beta_y}{\sqrt{n}}\ip{\blh_y}{\blTheta^\perp_y \blthetatl_y} + \frac{1}{\sqrt{p}}\ip{\bllambda_y}{\blrhotl_y - \blthetatl_y}\right]
    \end{aligned}
\end{equation}

Next, we explicitly solve the minimization over $\blTheta^\perp_y \blthetatl_y$ and let $\alpha_y = \normt{\blTheta^\perp_y \blthetatl_y}$ to get 

\begin{equation}
    \begin{aligned}
        &\min_{\substack{\blrhotl_y, \blTheta_y\blthetatl_y \in \R^p \\ \blv_y \in \R^{n_y} \\ \alpha_y \in \R_+}}  \max_{\substack{\blw, \bllambda_y \in \R^p \\ \beta_y \in \R_+}} \quad r \normt{\blSigma_1^{-1/2}\blrhotl_1}^2 + \ip{\blw}{\blSigma_1^{-1/2}\blrhotl_1 - \blSigma_{-1}^{-1/2}\blrhotl_{-1}}\\
        &\hspace{4em} +  \sum_{y \in \{-1,1\}} \left[\frac{1}{n}\ones^\top\scrL(\blv_y) + \frac{\beta_y}{\sqrt{n}}\normt*{-\blv_y + \ip{\blTStl_y}{\blTheta_y\blthetatl_y}\ones + \blZ_y\blTheta_y\blthetatl_y + \blg_y \alpha_y}\right.\\
        &\hspace{8em} \left. -\alpha_y\normt*{-\frac{\blTheta_y^\perp \bllambda_y}{\sqrt{p}} + \frac{\beta_y\blTheta_y^\perp \blh_y}{\sqrt{n}}} + \frac{1}{\sqrt{p}}\ip{\bllambda_y}{\blTheta_y(\blrhotl_y - \blthetatl_y} + \frac{1}{\sqrt{p}}\ip{\bllambda_y}{\blTheta^\perp_y \blrhotl_y}\right].
    \end{aligned}
\end{equation}
Note that this step requires switching the $\min$ and $\max$ (which is allowed by the compactification argument discussed earlier).
The next step is to rewrite the $\normt{\cdot}$ as $\normt{\cdot}^2$ using the trick $x = \min_{\tau \geq 0} \frac{x^2}{2 \tau} + \frac{\tau}{2}$ . This results in new decision variables $\tau_y, \eta_y \geq 0$:

\begin{equation}
    \begin{aligned}
        &\min_{\substack{\blrhotl_y, \blTheta_y\blthetatl_y \in \R^p \\ \blv_y \in \R^{n_y} \\ \alpha_y, \tau_y \geq 0}}  \max_{\substack{\blw, \bllambda_y \in \R^p \\ \beta_y, \gamma_y \geq 0}} \quad r \normt{\blSigma_1^{-1/2}\blrhotl_1}^2 + \ip{\blw}{\blSigma_1^{-1/2}\blrhotl_1 - \blSigma_{-1}^{-1/2}\blrhotl_{-1}}\\
        &\hspace{4em} +  \sum_{y \in \{-1,1\}} \left[\frac{1}{n}\ones^\top\scrL(\blv_y) + \frac{\beta_y}{2\tau_y n}\normt*{-\blv_y + \ip{\blTStl_y}{\blTheta_y\blthetatl_y}\ones + \blZ_y\blTheta_y\blthetatl_y + \blg_y \alpha_y}^2\right.\\
        &\hspace{6em} + \frac{\beta_y \tau_y}{2}\left. -\frac{\alpha_y}{2\gamma_y p}\normt*{-\frac{\blTheta_y^\perp \bllambda_y}{\sqrt{p}} + \frac{\beta_y\blTheta_y^\perp \blh_y}{\sqrt{n}}}^2 -\frac{\alpha_y \gamma_y}{2} + \frac{1}{\sqrt{p}}\ip{\bllambda_y}{\blTheta_y(\blrhotl_y - \blthetatl_y} + \frac{1}{\sqrt{p}}\ip{\bllambda_y}{\blTheta^\perp_y \blrhotl_y}\right].
    \end{aligned}
\end{equation}

Then, we optimize over $\bllambda_y$. This step involves a few intermediate simplifications (completing the square and separately optimizing over $\blTheta \bllambda$ and $\blTheta^{\perp} \bllambda$ separately. These steps are identical to those in \cite{taheri2020asymptotic}, and hence the details are omitted here. This step results in the additional constraint $\mu_y = \frac{\ip{\blTStl_y}{\blrhotl_y}}{\normt{\blTStl_y}^2}$, which enforces $\blTheta_y \blthetatl_y = \blTheta_y \blrhotl_y$. Again using strong duality, we write this as a maximization over the dual variables $\eta_y$ and ultimately obtain

\begin{equation}
    \begin{aligned}
        &\min_{\substack{\blrhotl_y\in \R^p \\ \blv_y \in \R^{n_y} \\ \alpha_y, \tau_y \in \R_+ \\ \mu_y \in \R}}  \max_{\substack{\blw \in \R^p \\ \beta_y, \gamma_y \in \R_+ \\ \eta_y \in \R}} \quad r \normt{\blSigma_1^{-1/2}\blrhotl_1}^2 + \ip{\blw}{\blSigma_1^{-1/2}\blrhotl_1 - \blSigma_{-1}^{-1/2}\blrhotl_{-1}}\\
        &\hspace{4em} +  \sum_{y \in \{-1,1\}} \left[\frac{1}{n}\ones^\top\scrL(\blv_y) + \frac{\beta_y}{2 \tau_y n}\normt*{-\blv_y + \ip{\blTStl_y}{\blTheta_y\blthetatl_y}\ones + \blZ_y\blTheta_y\blthetatl_y + \blg_y \alpha_y}^2\right.\\
        &\hspace{8em} \left. + \frac{\beta_y \tau_y}{2} - \frac{\alpha_y \gamma_y}{2} + \frac{\gamma_y}{2p \alpha_y} \normt*{\blrhotl_y \sqrt{p} + \frac{\alpha_y \beta_y}{\gamma_y \sqrt{\delta}} \blh_y}^2 - \frac{\mu_y^2 \gamma_y \normt{\blthetatl_y}^2}{2\alpha_y} - \frac{\alpha_y\beta_y^2 \norm{\blTheta_y^{\perp} \blh_y}^2}{2 n \gamma_y}\right.\\
        &\hspace{8em} \left. - \frac{\alpha_y \beta_y^2}{2 p \gamma_y \delta}\normt{\blTheta_y\blh_y}^2 - \frac{2 \beta_y}{\sqrt{n}}\ip{\blh}{\blTheta_y \blrhotl_y} + \eta_y\parens*{\mu_y - \frac{\ip{\blTStl_y}{\blrhotl_y}}{\normt{\blTStl_y}^2}} \right]
    \end{aligned}
\end{equation}

Note here, as in \cite{taheri2020asymptotic}, that the terms $\frac{\alpha_y \beta_y^2}{2 p \gamma_y \delta}\normt{\blTheta_y\blh_y}^2$ and $\frac{2 \beta_y}{\sqrt{n}}\ip{\blh}{\blTheta_y \blrhotl_y}$ tend to $0$ in the limit.

Next, optimization over $\blv_1$ and $\blv_{-1}$ is done by noting that
\begin{align*}
        &\min_{\blv_1, \blv_{-1}} \sum_{y \in \{-1,1\}} \left[\frac{1}{n}\ones^\top\scrL(\blv_y) + \frac{\beta_y}{2 \tau_y n}\normt*{-\blv_y + \ip{\blTStl_y}{\blTheta_y\blthetatl_y}\ones + \blZ_y\blTheta_y\blthetatl_y + \blg_y \alpha_y}^2\right]\\
        &= \sum_{y \in \{-1,1\}} \pi_y \min_{\blv_y \in \R^{n_y}}  \left[\frac{1}{n_y}\ones^\top\scrL(\blv_y) + \frac{\beta_y}{2 \tau_y n_y}\normt*{-\blv_y + \ip{\blTStl_y}{\blTheta_y\blthetatl_y}\ones + \blZ_y\blTheta_y\blthetatl_y + \blg_y \alpha_y}^2\right]\\
        &= \sum_{y \in \{-1,1\}} \pi_y \frac{1}{n_y} \sum_{i=1}^{n_y} \scrM_{\scrL}^{(1)}\parens*{\ip{\blTStl_y}{\blTheta_y\blthetatl_y} + \ip{\blz_{y,i}}{\blTheta_y \blthetatl_y} + \alpha_y \blg_{y,i} ; \frac{\tau_y}{\beta_y}}\\
        &= \sum_{y \in \{-1,1\}} \pi_y \frac{1}{n_y} \sum_{i=1}^{n_y} \scrM_{\scrL}^{(1)}\parens*{\mu_y\normt{\blthetatl_y}^2 + \ip{\blz_{y,i}}{\blTheta_y \blthetatl_y} + \alpha_y \blg_{y,i} ; \frac{\tau_y}{\beta_y}}
\end{align*}
Finally, we optimize over $\blrhotl$. This step differs from that in \cite{taheri2020asymptotic} since it accounts for the extra constraint that $\blSigma_1^{-1/2}\blrhotl_1 = \blSigma_{-1}^{-1/2}\blrhotl_{-1}$. 

 \begin{align*}
     &\min_{\substack{\blrhotl_1, \blrhotl_{-1} \\ \text{s.t. } \blSigma_1^{-1/2} \blrhotl_1 = \blSigma_{-1}^{-1/2} \blrhotl_{-1}}}  r \normt{\blSigma_1^{-1/2}\blrhotl_1}^2 + \sum_{y} \brackets*{\frac{\gamma_y}{2p \alpha_y} \normt*{\blrhotl_y \sqrt{p} + \frac{\alpha_y \beta_y}{\gamma_y \sqrt{\delta}} \blh_y}^2 - \frac{\eta_y \ip{\blTStl_y}{\blrhotl_y}}{\normt{\blTStl_y}^2}}\\
     &= \min_{\substack{\blrhotl_1, \blrhotl_{-1} \\ \text{s.t. } \blSigma_1^{-1/2} \blrhotl_1 = \blSigma_{-1}^{-1/2} \blrhotl_{-1}}}  r \normt{\blSigma_1^{-1/2}\blrhotl_1}^2 + \sum_{y} \brackets*{\frac{\gamma_y}{2p \alpha_y} \normt*{\blrhotl_y \sqrt{p} + \frac{\alpha_y \beta_y}{\gamma_y \sqrt{\delta}} \blh_y - \frac{\eta_y \alpha_y \sqrt{p}}{\gamma_y \normt{\blTStl_y}^2}\blTStl_y}^2 - \frac{\eta_y^2 \alpha_y}{2 \gamma_y \normt{\blTStl_y}^2}}\\
 \end{align*}

Here, we have rewritten the term involving $\blw$ as an explicit constraint on the $\blrhotl_y$, using strong duality. Now, note that this is equivalent to optimizing over a single variable $\blrho = \blSigma_1^{-1/2} \blrhotl_1 = \blSigma_{-1}^{-1/2} \blrhotl_{-1}$, so we can write this as

\begin{align*}
    & \min_{\blrho}  r \normt{\blrho}^2 + \sum_{y} \brackets*{\frac{\gamma_y}{2p \alpha_y} \normt*{\blSigma_y^{1/2} \parens*{\blrho \sqrt{p} + \frac{\alpha_y \beta_y}{\gamma_y \sqrt{\delta}} \blSigma_y^{-1/2} \blh_y - \frac{\eta_y \alpha_y \sqrt{p}}{\gamma_y \normt{\blTStl_y}^2}\blSigma_y^{-1} \blTS}}^2 - \frac{\eta_y^2 \alpha_y}{2 \gamma_y \normt{\blTStl_y}^2}}\\
    &= \frac{r}{p} \sum_{i=1}^p \scrM^{(2)}_{\ell_2^2}\parens*{\frac{\alpha_y \beta_y}{\gamma_y \sqrt{\delta \lambda_{y,i}}} \blh_{y,i} + \frac{\eta_y \alpha_y \sqrt{p}}{\gamma_y \normt{\blTStl_y}^2 \lambda_{y,i}} \blTS_i ; \frac{\alpha_y r}{\gamma_y \lambda_{y,i}}} - \sum_{y} \frac{\eta_y^2 \alpha_y}{2 \gamma_y \normt{\blTStl_y}^2}
\end{align*}

We have now expressed the AO entirely in terms of a scalar optimization problem. We then can consider the asymptotics of each term in the objective function of the scalarized AO, using the fact that $\normt{\blTStl_y} \to \zeta_y$ and the Moreau envelope terms converge in probability to their expectations as $n, p \to \infty$. This shows pointwise convergence of the objective function. We can use the fact that pointwise convergence of continuous, convex functions over compact sets is uniform (e.g., Cor II.1 in \cite{andersen1982cox}) to conclude that the optimal value of the AO converges to the optimal value of the scalar min-max problem. Hence, by the CGMT, the optimal value of the original ERM problem (the PO) also converges in probability to this value.

We now need to extend this to a statement about the \textit{test error}, as given in Lemma \ref{lem:poe}. By Lemma 7 in \cite{taheri2020asymptotic}, the desired result follows if we can show that the solution $\blthetatl_y$ of the POE satisfies
\[
\frac{\ip{\blthetatl_y}{\blTStl_y}}{\normt{\blTStl_y}^2} \overset{P}{\to} \mu_y^*, \quad \normt{\blTheta_y^{\perp} \blthetatl_y} \overset{P}{\to}
 \alpha_y^*,
 \]
 where $\mu_y^*$ and $\alpha_y^*$ are the optimal decision variables from \ref{thm:main_thm}. The proof of these two statements are identical, so we only show the latter. First fix $\epsilon > 0$ and define the sets $\scrS_y = \braces*{\blthetatl_y \colon \abs*{\normt{\blTheta_y^{\perp} \blthetatl_y} -
 \alpha_y^*} < \epsilon}$. Now we consider the same AO but with the additional constraint $\blthetatl_y \in \scrS_y^C$. By assumption that the scalar min-max in Theorem \ref{thm:main_thm} has a unique solution, this can only strictly increase the optimal value with probability approaching 1. Hence, we can apply the third part of Theorem 3 in \cite{dhifallah2021inherent} to conclude that $\P \braces{\blthetatl_y \notin \scrS_y} \to 0$ and the desired convergence in probability to $\alpha_y^*$ holds. Substituting these values into the expression for the POE in Lemma 7 of \cite{taheri2020asymptotic}, we can conclude that the per-class POE converges to the desired quantity. 

\end{proof}

\subsection{Additional details for the numerical simulations}\label{app:sim_details}
In this section, we provide additional details about the results displayed in Figure \ref{fig:theory}. 

In each of the three spectral imbalance settings, we apply Theorem \ref{thm:main_thm} with $\pi_y = 0.5$, the overparameterization ratio $\delta = 2$, regularization parameter $r=0.5$ and the loss as the squared hinge loss, $\scrL(t) = \max(0, 1-t)^2$. The scalar min-max problem is solved using gradient descent/ascent with learning rate 0.01. The choice of $\Pi$ used for each of the three settings is given below:

\begin{enumerate}
    \item \textbf{Impact of eigenvalue scaling (Setting (A)):} 
    \[
    (\sqrt{p} \bltheta^*_i, \lambda_{i}^{(1)}, \lambda_{i}^{(-1)}) \simiid 
    \begin{cases} 
      (1, 2s, 2) & \text{with prob. } 0.5 \\
      (1, 0.5s, 0.5) & \text{with prob. } 0.5   \end{cases}
    \]
    \item \textbf{Impact of eigenvalue decay (Setting (B)):}
    \[
    (\sqrt{p} \bltheta^*_i, \lambda_{i}^{(1)}, \lambda_{i}^{(-1)}) \simiid 
    \begin{cases} 
      (1, 2, 2) & \text{with prob. } 0.5p \\
      (1, 2, 0.5) & \text{with prob. } 0.5p \\
      (1, 0.5, 2) & \text{with prob. } 0.5(1-p) \\
      (1, 0.5, 0.5) & \text{with prob. } 0.5(1-p)   \end{cases}
    \]
    \item \textbf{Impact of alignment with target signal (Setting (C)):}
    \[
    (\sqrt{p} \bltheta^*_i, \lambda_{i}^{(1)}, \lambda_{i}^{(-1)}) \simiid 
    \begin{cases} 
      (1, 2, 0.5) & \text{with prob. } 0.5 \\
      (a, 0.5, 2) & \text{with prob. } 0.5   \end{cases}
    \]
\end{enumerate}
In each case, the numerical simulations are averaged over $50$ independent draws of $n=1000$ training examples (so $p=500$) from the GMM. For each trial, we solve the ERM objective function using gradient descent. 

\subsection{The interplay of sample imbalance and spectral imbalance}\label{app:sampleimbalance}

In Figure \ref{fig:theory}, we aimed to isolate the effect of spectral imbalance, so we assume that the training data is \textit{sample-balanced}, i.e., $\pi_1 = \pi_{-1} = \frac{1}{2}$ and an equal proportion of the training data comes from each class. The asymptotic predictions in Theorem \ref{thm:main_thm}, however, apply for general $\pi_y$, allowing for a careful study of the interplay between these aspects of the model setting.

In Figure \ref{fig:appendix_dataimb}, we provide the exact asymptotic predictions for this interplay in the three spectrally-imbalanced settings we introduce in Section \ref{sec:theory-insights}. These results reveal a nuanced relationship between these factors: for a fixed amount of spectral imbalance, data imbalance can either reduce or exacerbate the class gap, depending on the ``direction'' of the imbalance. Interestingly, we find that this relationship depends crucially on the type of spectral imbalance in the dataset. For example, in setting (A), sample imbalance seems to have little asymptotic effect on the class gap for a fixed amount of spectral imbalance. By contrast, in settings (B) and (C), the sample imbalance and spectral imbalance both have a pronounced role in determining the class gap.

\begin{figure}[!t]
    \begin{center}
    \includegraphics[width=\textwidth]{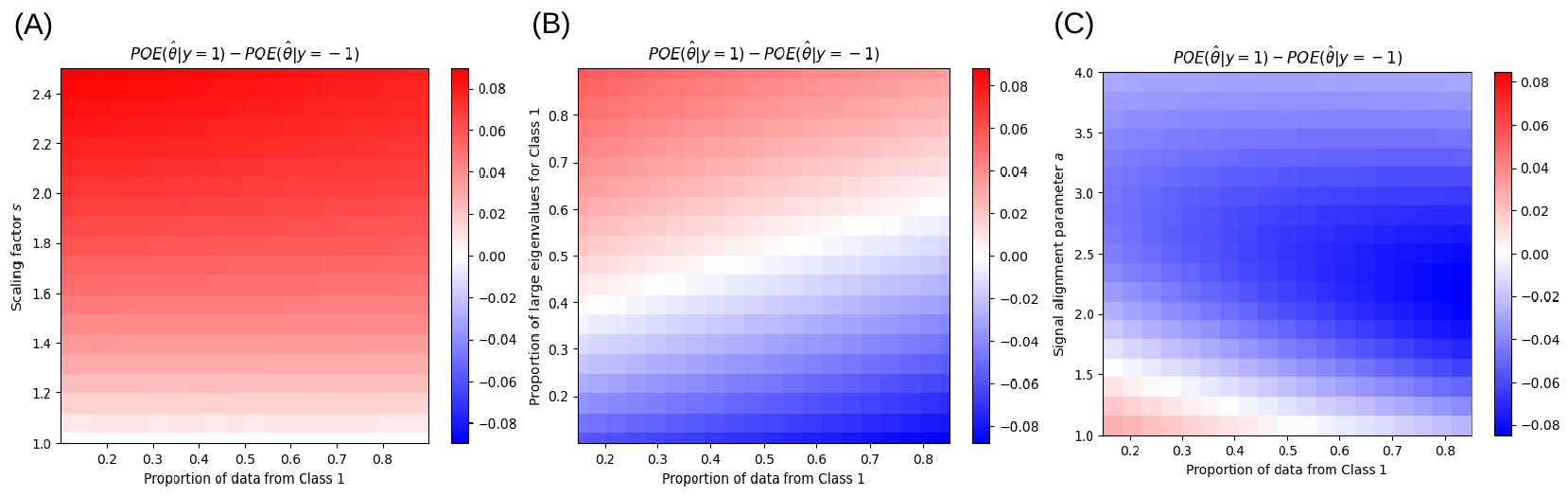}
    \caption{\textit{Interplay of sample and spectral imbalance:} Heatmap of the class gap across different amounts of sample imbalance and spectral imbalance. Settings (A), (B), and (C) correspond to the same three types of spectral imbalance considered in Section \ref{sec:theory-insights} and the simulation details provided in Appendix \ref{app:sim_details}.
    \label{fig:appendix_dataimb}}
    \end{center}
\end{figure}

\section{Additional Experimental Details}\label{app:exp}

In Section~\ref{app:exp_visual}, we first provide additional visualizations to supplement the visualizations presented in the main text. 
In Section~\ref{app:exp_pcc_score}, we record the Pearson correlation coefficient score and other statistical metrics for each encoder. 
In Section~\ref{app:network_details}, we record the performance and details of the pre-trained models studied.
In Section~\ref{app:imagenetc_details}, we record the details of the ImageNet-C dataset created and used in Section~\ref{exp:data_augmentation} and also expand on the results in Section~\ref{exp:data_augmentation}.

\subsection{Additional visualizations}
\label{app:exp_visual}

In Figure~\ref{fig:appendix_direct}, we visualize the eigenspectrum in VIT-B as an additional example to complement Figure~\ref{fig:motivation}(A), which visualized the eigenspectrum for ResNet-50.
Specifically, the left panel plots the eigenspectrum on regular scale (as in Figure~\ref{fig:motivation}(A)) for VIT-B; the right panel plots the eigenspectrum on log-log scale to reveal the power law behavior, which was also observed and analyzed for ResNet-50 in Section~\ref{sec:experiments}. 

\begin{figure*}[!ht]
    \centering
    \includegraphics[width=0.8\textwidth]{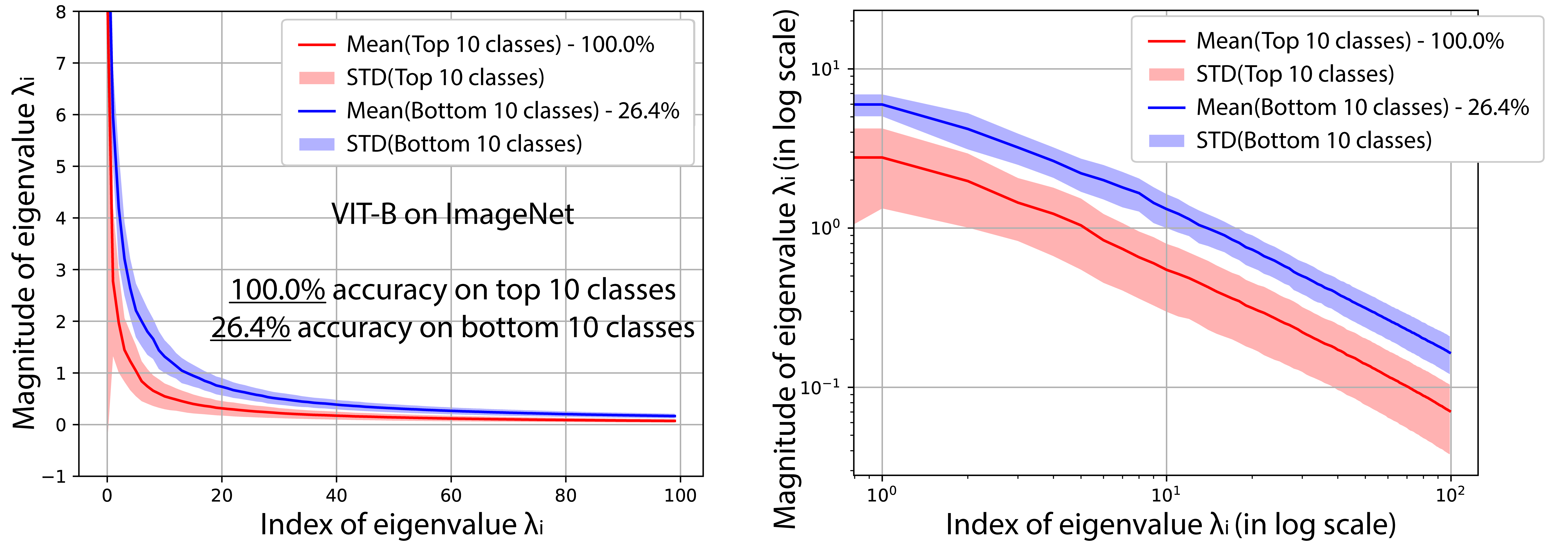}
    \vspace{-3mm}
    \caption{{\em Eigenspectrum visualization plots in normal scale and log scale for VIT-B.}}
    \label{fig:appendix_direct}
\end{figure*}

Additionally, in Figure~\ref{fig:appendix_direct_more}, we show the direct visualization of eigenspectrum in other encoders. The distributional differences across classes are consistent and strong in all observed representation spaces.

\begin{figure*}[!ht]
    \centering
    \includegraphics[width=\textwidth]{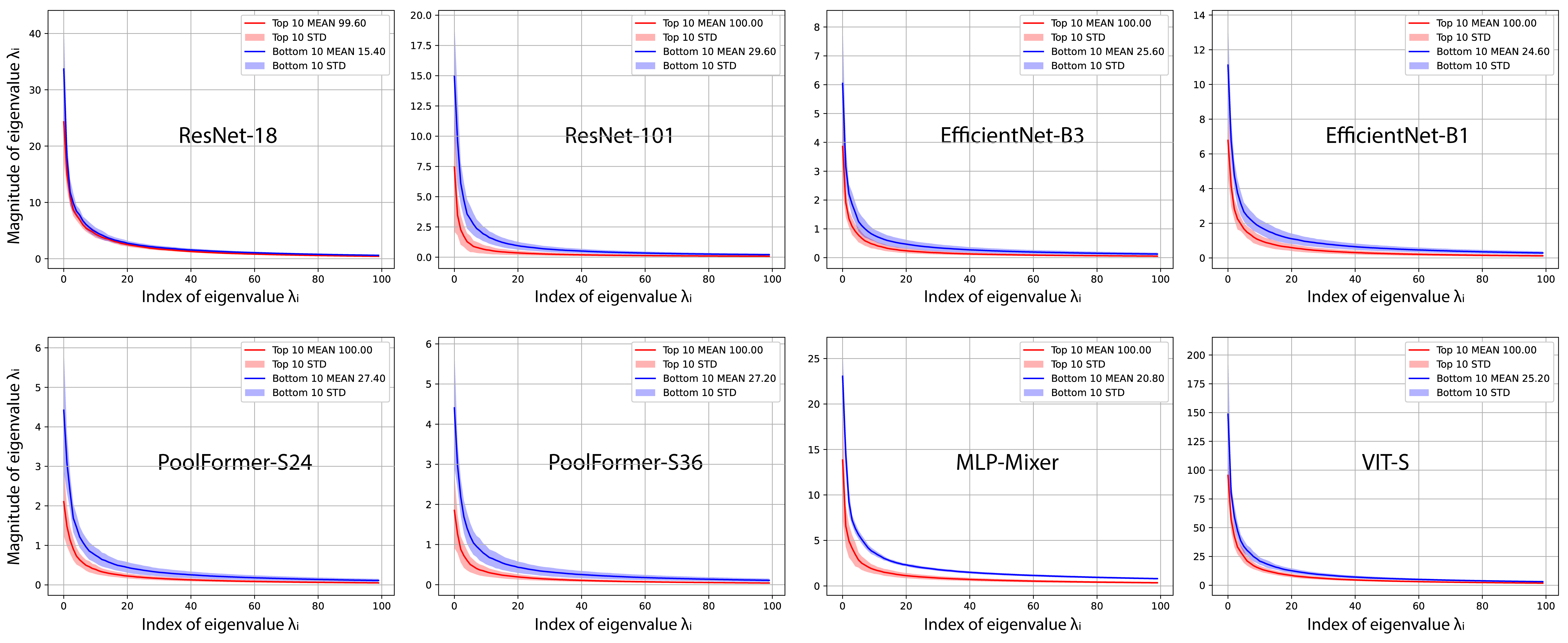}
    \vspace{-3mm}
    \caption{{\em Eigenspectrum visualization plots for other encoders.}}
    \label{fig:appendix_direct_more}
\end{figure*}

In Figure~\ref{fig:appendix_dist_compare}, we show that the distribution of eigenvalues across encoders are different for additional eigenvalue positions. The same figure as in Figure~\ref{fig:motivation}(B) (which was plotted for $\lambda_5$) is plotted for $\lambda_{50}$ and $\lambda_{100}$.
We observed similar trends for all other eigenvalue indexes as well. 

\begin{figure*}[!ht]
    \centering
    \includegraphics[width=0.8\textwidth]{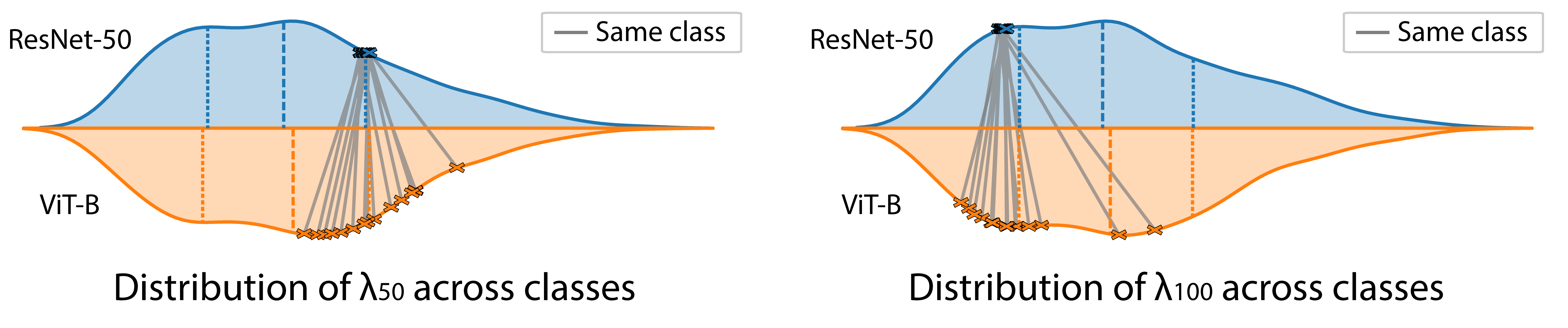}
    \vspace{-3mm}
    \caption{{\em Comparing the eigenvalue distributions across encoders for $\lambda_{50}$ and $\lambda_{100}$ (for all classes $C=1000$).}}
    \label{fig:appendix_dist_compare}
\end{figure*}

Finally, in Figure~\ref{fig:appendix_pcc}, we show the Pearson correlation coefficient plots corresponding to Figure~\ref{fig:correlation}(B) (which corresponded to ResNet-50) for all the other models.

\begin{figure*}[!ht]
    \centering
    \includegraphics[width=\textwidth]{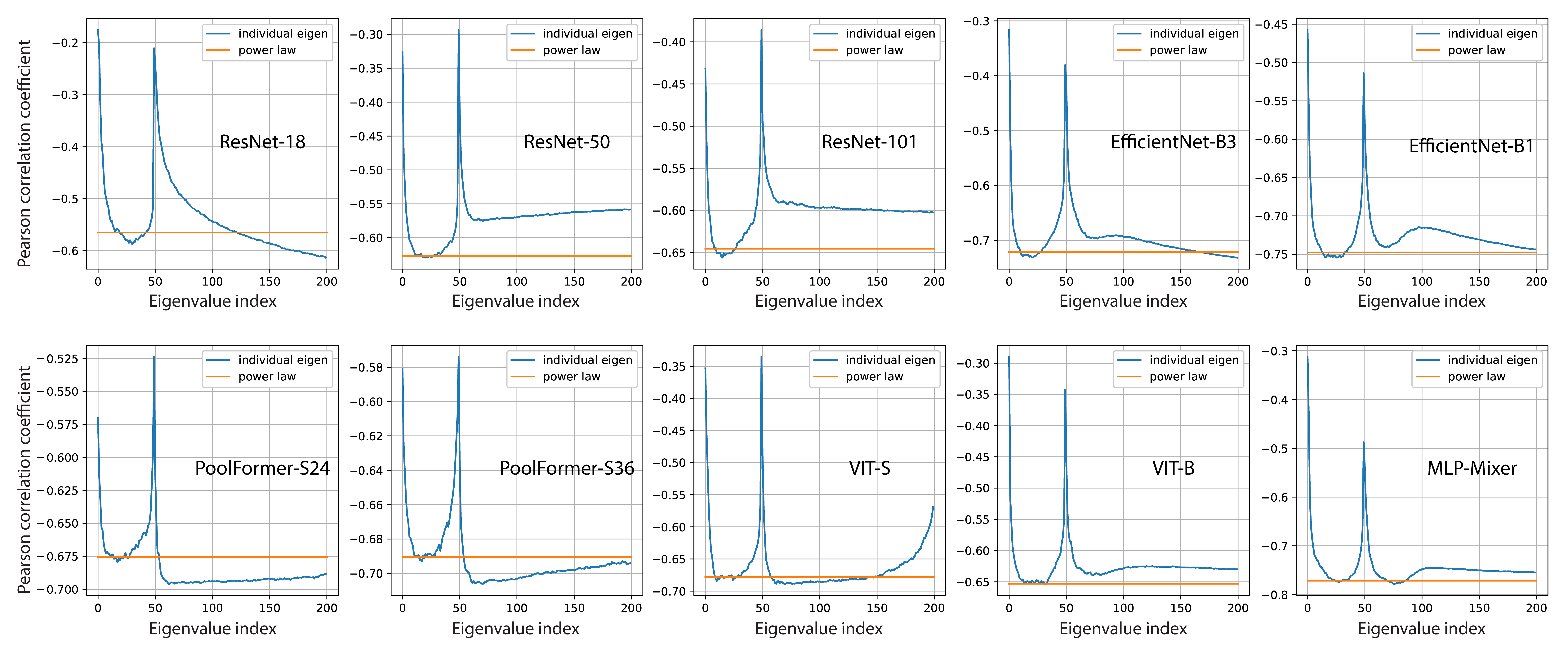}
    \vspace{-3mm}
    \caption{{\em Correlation plots for other encoders.}}
    \label{fig:appendix_pcc}
\end{figure*}

\subsection{Correlation analysis of all the encoders}
\label{app:exp_pcc_score}

First, we systematically record the average accuracy of the bottom 10, bottom 100, top 100, and top 10 classes of each encoder in Table~\ref{table:bias_details} to highlight the extent of the class bias.

\begin{table}[h!]
\caption{{The prevalence of class bias across encoders.}
}
\label{table:bias_details}
\begin{center}
\begin{small}
\begin{sc}
\begin{tabular}{lcccc}
\toprule
Model & Bottom 10 & Bottom 100  & Top 100 & Top 10 \\
\midrule
Densenet & 21.40\% & 40.76\% & 96.44\% & 99.80\% \\
EfficientNet-B1 & 24.60\% & 44.00\% & 97.64\% & 100.00\% \\
EfficientNet-B3 & 25.60\% & 44.64\% & 98.00\% & 100.00\% \\
Resnet-18 & 15.40\% & 34.06\% & 94.34\% & 99.60\% \\
Resnet-50 & 28.80\% & 46.94\% & 98.40\% & 100.00\% \\
Resnet-101 & 29.60\% & 49.44\% & 98.64\% & 100.00\% \\
VIT-S & 25.20\% & 41.22\% & 96.02\% & 100.00\% \\
VIT-B & 26.40\% & 48.86\% & 98.54\% & 100.00\% \\
MLPMixer & 20.80\% & 38.52\% & 96.08\% & 100.00\% \\
Poolformer-S24 & 27.40\% & 48.10\% & 98.60\% & 100.00\% \\
Poolformer-S36 & 27.20\% & 49.04\% & 98.80\% & 100.00\% \\
\midrule
AVG & 24.76\% & 44.14\%  &  97.41\% & 99.95\% \\
\bottomrule
\end{tabular}
\end{sc}
\end{small}
\end{center}
\vskip -0.1in
\end{table}

Second, we record the Pearson correlation coefficient (PCC) between per-class accuracy and class-dependent eigenspectrum properties for each encoder in Table~\ref{table:pcc_details}.
Aside from the individual eigenvalue scaling and power law offset, we also record the power law slope and two effective ranks that measure the \emph{rate} of eigenvalue decay in different ways. The definition of the studied two effective ranks are as below:


\begin{definition}[\textbf{Effective Rank}]\label{eff_ranks}
    For any covariance matrix (spectrum) $\blSigma$, we define its two effective ranks:
    \begin{equation}
    \rho_{k}({\blSigma})=\frac{\sum_{i>k} \lambda_{i}}{\lambda_{k+1}} ,~~R_k({\blSigma}):=\frac{( \sum_{i>k}\lambda_i)^2}{\sum_{i>k}\lambda_i^2}.
    \end{equation}
\end{definition}
This definition was proposed in~\cite{bartlett2020benign} and has been shown to constitute a ``sufficient" statistic to explain average generalization error across regression and classification tasks~\cite{bartlett2020benign,tsigler2020benign,muthukumar2020class,wang2021binary}.
We examine the relationship between effective ranks for $k = 0$ and defer a more detailed study for other values of $k$ to future work.

\begin{table}[h!]
\caption{{Examining the correlation between spectral properties and per-class accuracy on different pre-trained encoders.}
}
\label{table:pcc_details}
\begin{center}
\begin{small}
\begin{sc}
\begin{tabular}{lccccc}
\toprule
Model & Max Corr. & PL Corr $a$. & PL Corr $b$. & Effective Rank 1 & Effective Rank 2 \\
\midrule
Densenet & -0.6677 & -0.6636  & -0.0598 & -0.5107 & -0.5384 \\
EfficientNet-B1 & -0.7599 & -0.7476  & -0.0107 & -0.6283 & -0.6131 \\
EfficientNet-B3 & -0.7505 & -0.7210  & 0.0007 & -0.6209 & -0.5816 \\
Resnet-18    & -0.6229 & -0.5649 & 0.0968 & -0.4434 & -0.4722 \\
Resnet-50    & -0.6293 & -0.6270  & 0.0447 & -0.4943 & -0.4952 \\
Resnet-101   & -0.6560 & -0.6454  & 0.0170 & -0.4869 & -0.4896 \\
VIT-S      & -0.6891 & -0.6783  & -0.0571 & -0.5195 & -0.5399 \\
VIT-B      & -0.6536 & -0.6531  & -0.1493 & -0.4824 & -0.4934 \\
MLPMixer   & -0.7785 & -0.7715  & 0.0460 & -0.7074 & -0.6707 \\
Poolformer-S24  & -0.6961 & -0.6755  & 0.0708 & -0.5544 & -0.5895 \\
Poolformer-S36  & -0.7062 & -0.6904  & 0.0006 & -0.5652 & -0.5638 \\
\midrule
AVG & -0.6918 & -0.6762  & -0.0002 & -0.5466 & -0.5497 \\
\bottomrule
\end{tabular}
\end{sc}
\end{small}
\end{center}
\vskip -0.1in
\end{table}
Interestingly, we observe across models that the spectral metrics that measure \emph{relative rate} of eigenvalue decay --- such as the power law decay parameter $b$ and the effective ranks --- do not strongly correlate with class-wise accuracy, while the spectral metrics that are more absolute --- i.e. the individual eigenvalues and the power law parameter $a$ --- do strongly correlate.

\subsection{Correlation analysis on smaller datasets}
\label{app:smaller_datasets}
Aside from ImageNet, we also perform the same correlation analysis on smaller-scale datasets like CIFAR-10 and FASHION-MNIST to demonstrate the versatility of the proposed method. As shown in Table~\ref{table:rebuttal-small}, the negative correlation is stronger in smaller-scale datasets with fewer classes.



\begin{table}[h!]
\caption{{Examining the correlation between spectral properties and per-class accuracy on smaller datasets.}
}
\label{table:rebuttal-small}
\begin{center}
\begin{small}
\begin{sc}
\begin{tabular}{lccc}
\toprule
Model & CIFAR-10 & Fashion-MNIST & ImageNet \\
\midrule
Densenet & -0.8750 & -0.8045 & -0.6636 \\
EfficientNet-B1 & -0.8104 & -0.7330 & -0.7476 \\
EfficientNet-B3 & -0.8906 & -0.5591 & -0.7210 \\
Resnet-18    & -0.8311 & -0.7634 & -0.5649 \\
Resnet-34    & -0.8334 & -0.6677 & -0.5772 \\
Resnet-50   & -0.8859 & -0.8605 & -0.6270 \\
VIT-S      & -0.7293 & -0.7922 & -0.6783 \\
VIT-B      & -0.8049 & -0.4837 & -0.6531 \\
MLPMixer   & -0.7228 & -0.9127 & -0.7715 \\
Poolformer-S24  & -0.8477 & -0.6522 & -0.6755 \\
Poolformer-S36  & -0.8173 & -0.7057 & -0.6904 \\
\midrule
AVG & -0.8226 & -0.7213  & -0.6700 \\
\bottomrule
\end{tabular}
\end{sc}
\end{small}
\end{center}
\vskip -0.1in
\end{table}



\subsection{Details of pretrained networks}
\label{app:network_details}
In this section we record the details of the pre-trained networks that we used and evaluated to allow for reproducibility of our results.
\begin{itemize} 
    \item \textbf{Residual Networks.} ResNet is instrumental in many visual recognition tasks, as it addresses the vanishing gradient problem,
    allowing very deep networks to be trained effectively. 
    In this paper, we used the pre-trained weights for ResNet from TorchVision. We list their performance, model card name, and parameters below.
    
    \begin{itemize}
    \item For ResNet-101, we used the v2 weight from the new training recipe of torchvision. The Top1 accuracy is 81.886\% and Top5 accuracy is 95.78\% on ImageNet-1K, with 44.5M parameters. Its representation dimensionality is equal to 2048.
    \item For ResNet-50, we used the v2 weight from the new training recipe of torchvision. The Top1 accuracy is 80.858\% and Top5 accuracy is 95.434\% on ImageNet-1K, with 25.6M parameters. Its representation dimensionality is equal to 2048.
    \item For ResNet-18, we used the default weight (or v1 weight from torchvision). The Top1 accuracy is 69.758\% and Top5 accuracy is 89.078\% on ImageNet-1K, with 11.7M parameters. Its representation dimensionality is equal to 512.
    \end{itemize}

\end{itemize}
\begin{itemize}

    \item \textbf{Improved CNNs.} Many research works propose variants of convolutional architectures to improve ResNets. Among these efforts, DenseNet \cite{huang2017densely} and EfficientNet \cite{tan2019efficientnet} are two representative architectures that are more parameter-efficient.
    They design different strategies to connect layers differently to improve the parameter-efficiency as well as information flow inside the networks.
    DenseNet improves the flow of information and gradients throughout the network by densely connecting all layers directly with each other. EfficientNet designs a compound coefficient that uniformly scales the depth, width, and resolution of CNNs and thus improves performance with fewer parameters.
    \begin{itemize}
    \item For DenseNet, we used the weights from timm (model card \texttt{densenet121.ra\_in1k}). The Top1 accuracy is 75.57\% and Top5 accuracy is 92.61\%, with 8.0M parameters. Its representation dimensionality is equal to 1024.
    \item For EfficientNet-B3, we used the weights from timm (model card \texttt{efficientnet\_b3.ra2\_in1k}). The Top1 accuracy is 77.60\%  and Top5 accuracy is 93.59\%,  with 12.2M parameters. 
     Its representation dimensionality is equal to 1536.
    \item For EfficientNet-B1, we used the weights from timm (model card \texttt{efficientnet\_b1.ft\_in1k}), The Top1 accuracy is 78.54\% and  Top5 accuracy is 94.38\%,  with 7.8M parameters. Its representation dimensionality is equal to 1280.
    \end{itemize}
    
\end{itemize}
\begin{itemize}

    \item \textbf{Vision Transformers.} Recently, the transformer architecture 
    has achieved strong performance even in the vision domain \cite{dosovitskiy2020image}. The vision transformer splits an image into patches and uses self-attention to study the interaction of patches. By doing so, it aggregates global information at an early stage, creating significantly different representations than CNNs \cite{raghu2021vision}.

    \begin{itemize}
    \item For VIT-B, we used the weights from timm (\texttt{vit\_base\_patch16\_224.augreg2\_in21k\_ft\_in1k}), The Top1 accuracy is 81.10\% and Top5 accuracy is 95.72\%, with 86.6M parameters. Its representation dimensionality is equal to 768. We note that the model is pre-trained on ImageNet21k and later fine-tuned on ImageNet1k.
    \item For VIT-S, we used the weights from timm (\texttt{vit\_small\_patch16\_224.augreg\_in21k\_ft\_in1k}). The Top1 accuracy is 74.63\% and Top5 accuracy is 92.67\%, with 22.1M parameters. Its representation dimensionality is 
    equal to 384. We note that the model is pre-trained on ImageNet21k and later fine-tuned on ImageNet1k.
    \end{itemize}

\end{itemize}
\begin{itemize}

    \item \textbf{Transformers without Self-Attention.} Recent studies show that self-attention is not required for obtaining good representations in vision. We study in particular MLP-Mixer \cite{tolstikhin2021mlp} and PoolFormer \cite{yu2022metaformer}, which uses multi-layer perceptrons and 
    average pooling to mix information across patches, respectively. 

    \begin{itemize}
    \item For MLP-Mixer, we used weights from timm (model card \texttt{mixer\_b16\_224.goog\_in21k\_ft\_in1k}). The Top1 accuracy is 72.57\% and Top5 accuracy is 90.06\%, with 59.9M parameters. Its representation dimensionality is equal to 768. We note that it is pre-trained on ImageNet21k and later fine-tuned on ImageNet1k.
    \item For PoolFormer-S36, we used the weights from timm (model card \texttt{poolformerv2\_s36.sail\_in1k}). The Top1 accuracy is 81.54\% and Top5 accuracy is 95.69\%, with 30.8M parameters. Its representation dimensionality is equal to
    512.
    \item For PoolFormer-S24, we used the weights from timm (model card \texttt{poolformerv2\_s24.sail\_in1k}). The Top1 accuracy is 80.76\% and Top5 accuracy is 95.39\%, with 21.3M parameters. Its representation dimensionality is equal to 512.
    \end{itemize}
\end{itemize}




\subsection{Details of ImageNet-C results}
\label{app:imagenetc_details}

We created the ``augmented" ImageNet-C datasets based on \cite{hendrycks2019benchmarking}. The preprocessing plans (corruptions) we selected are:
\begin{itemize}
    \item Gaussian Noise, 
Shot Noise,
Impulse Noise,
Defocus Blur,
Motion Blur,
Zoom Blur,
Snow,
Fog,
Brightness,
Contrast,
Elastic,
Pixelate,
JPEG,
Speckle Noise,
Spatter, and Saturate.
\end{itemize}
Note that these corruptions correspond to the $16$ augmentations discussed in Section~\ref{exp:data_augmentation}.
While~\cite{hendrycks2019benchmarking} only augment the test dataset, we perform the image altering operations on the training set as well, as we wish to estimate the eigenvalues using the training data. 

\begin{figure*}[!ht]
    \centering
    \includegraphics[width=0.6\textwidth]{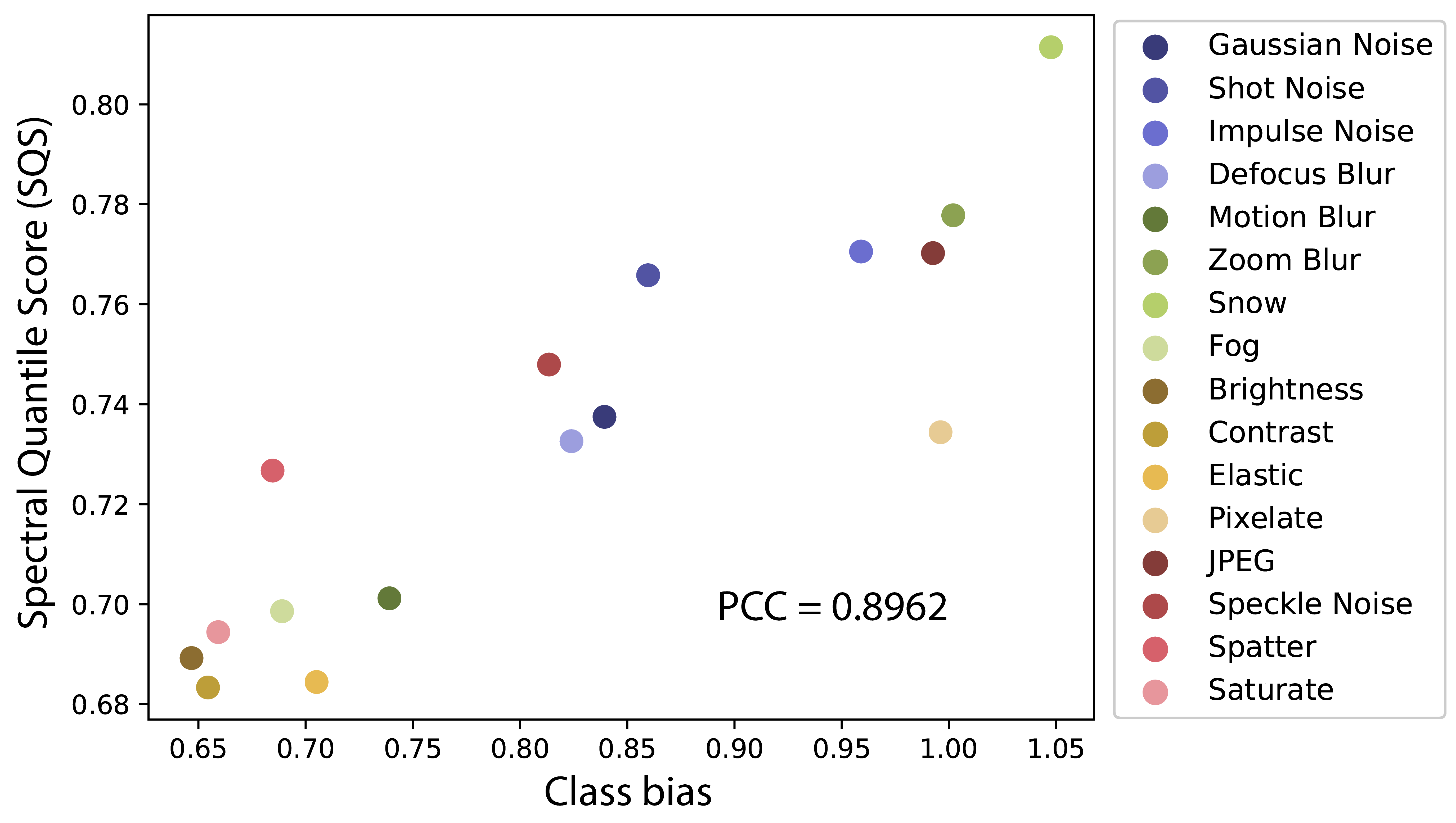}
    \vspace{-3mm}
    \caption{{\em Spectral Quantile Score (SQS) of different data manipulation plans.}}
    \label{fig:appendix_svs}
\end{figure*}

In Figure~\ref{fig:appendix_svs}, we show the Spectral Quantile Score (SQS) of applying the same encoder on different data with different data manipulation plans (as discussed in Section~\ref{exp:data_augmentation}). The PCC is 0.90, showing strong correlation (again) between class bias and the SQS.

For the statistical significance results for each class, we first compute the Pearson correlation coefficient $r$ for each class, and calculate the $p$-value as below:
\begin{equation}
    p = 2P(T_{n-2}>t), \text{ where } t=\frac{r \sqrt{n-2}}{\sqrt{1-r^2}}
\end{equation}
where $n$ is the number of samples, and $T_{n-2}$ is a Student $t$'s distribution with $n-2$ degrees of freedom.


\begin{table}[h]
\caption{Full improvement results on ImageNet-C.}
\label{table:imagenetc_rebuttal}
\vskip 0.15in
\begin{center}
\begin{small}
\begin{sc}

\begin{tabular}{c|ccccccccc}
\toprule
\multirow{2}{*}{Improvement} 
& Gaussian & Shot & Impulse & Defocus & Motion & Zoom & & & \\
& Noise & Noise & Noise & Blur & Blur & Blur & Snow & Fog & Brightness \\
\midrule
$\Delta$(AVG, Ours) & 0.019 & 0.022 & 0.031 & 0.019 & 0.017 & 0.039 & 0.040 & 0.024 & 0.034 \\
$\Delta$(MIN, Ours) & 0.044 & 0.048 & 0.068 & 0.045 & 0.044 & 0.080 & 0.081 & 0.056 & 0.074 \\
\bottomrule
\end{tabular}
\vspace{2mm}
\begin{tabular}{c|cccccccc}
\toprule
\multirow{2}{*}{Improvement} & & & & & Speckle & & & \multirow{2}{*}{\textbf{Imp. Avg.}} \\
& Contrast & Elastic & Pixelate & JPEG & Noise & Spatter & Saturate & \\
\midrule
$\Delta$(AVG, Ours) & 0.030 & 0.024 & 0.033 & 0.010 & 0.020 & 0.024 & 0.030 & 0.026 \\
$\Delta$(MIN, Ours) & 0.066 & 0.054 & 0.075 & 0.075 & 0.045 & 0.057 & 0.064 & 0.059 \\
\bottomrule
\end{tabular}
\end{sc}
\end{small}
\end{center}
\vskip -0.1in
\end{table}

\paragraph{Complete results} 
We repeat the experiments to systematically demonstrate that the proposed ensemble method works across many different combinations and types of augmentations. Here, for each version of ImageNet-C (out of 16 total versions), we apply the ensemble method based on the other 15 types of corruption, creating 16x15 = 240 runs. For each column, we report the average and best improvement in overall accuracy obtained by ensembling. As shown in Table~\ref{table:imagenetc_rebuttal}, across 240 runs, our proposed ensemble method can improve the average augmentation and worse augmentation by 2.6\% and 5.9\%, respectively.

\subsection{Comparing spectral imbalance and semantic scale}
\label{app:compare_volume}
The authors of \cite{ma2022delving} study the overall volume of each class (related to the product of eigenvalues of the class-covariances) and its correlation with class disparities. In contrast, our work reveals dependence on the eigenvalue distributions rather than the overall volume, which we believe provides a more fine-grained characterization of feature imbalances. 

To demonstrate the difference between the two approaches, we compare our results to those reported in \cite{ma2022delving} for the case of sample-balanced data. Interestingly, we obtain similar overall levels of predictive accuracy (as measured by the Pearson correlation coefficient between the spectra and the per-class performance) to their proposed inter-cluster corrected measure (S), which takes into account inter-cluster distances in addition to the eigenvalue volume (S’). In contrast, the spectral volume (S’) itself does not have as strong predictive performance. Thus, this shows that a more fine-grained analysis of the eigenvalues can have a big benefit when describing class gaps in the sample-balanced setting.

\begin{table}[h!]
\caption{{Direct comparison with the semantic scale method on CIFAR-10.}
}
\label{table:rebuttal_scale}
\begin{center}
\begin{small}
\begin{sc}
\begin{tabular}{ccc}
\toprule
Model & Method & Absolute correlation \\
\midrule
ResNet-18 & \cite{ma2022delving} (S’) & 0.5433 \\
ResNet-18 & \cite{ma2022delving} (Best S) & 0.7850 \\
ResNet-18 & Ours & \textbf{0.8311} \\
ResNet-18 & $\Delta$ & 0.0461 \\
\midrule
ResNet-34 & \cite{ma2022delving} (S’) & 0.5750 \\
ResNet-34 & \cite{ma2022delving} (Best S) & 0.8056 \\
ResNet-34 & Ours & \textbf{0.8334} \\
ResNet-34 & $\Delta$ & 0.0278 \\
\bottomrule
\end{tabular}
\end{sc}
\end{small}
\end{center}
\vskip -0.1in
\end{table}


\end{document}